\documentclass[reqno, a4paper]{amsart}
\usepackage[centering]{geometry}
\usepackage{amsmath,amssymb,amsthm}
\usepackage[scr=boondox]{mathalfa}
\usepackage[svgnames]{xcolor}
\usepackage[unicode, colorlinks=false, hidelinks]{hyperref}

\usepackage{enumerate} % [(i)] kann man machen oki
\usepackage{cleveref}

\usepackage{bbm}
\usepackage{booktabs}
\usepackage{graphicx}

\usepackage[shortlabels]{enumitem}
\usepackage{dsfont}
\usepackage{color}
\usepackage[utf8]{inputenc}
\usepackage[T1]{fontenc}
\usepackage{lmodern}
\usepackage{anyfontsize}
\usepackage{cleveref}
\usepackage{comment}
\usepackage{makecell}

\usepackage{tikz-cd}
\usepackage{algorithm}
\usepackage{algorithmic}
\usepackage{marginnote}
\usepackage{csvsimple}

\newcommand{\MB}[1]{\marginpar
  {\color{purple} 
    \small 
    \ifodd\value{page} 
      \raggedright % Right margin: align left
    \else 
      \raggedleft  % Left margin: align right
    \fi 
    \parbox{2.5cm}{M: #1}
  }
}
\newcommand{\JB}[1]{\marginpar
  {\color{magenta} 
    \small 
    \ifodd\value{page} 
      \raggedright % Right margin: align left
    \else 
      \raggedleft  % Left margin: align right
    \fi 
    \parbox{2.5cm}{J: #1}
  }
}
\newcommand{\LR}[1]{\marginpar
  {\color{MidnightBlue} 
    \small 
    \ifodd\value{page} 
      \raggedright % Right margin: align left
    \else 
      \raggedleft  % Left margin: align right
    \fi 
    \parbox{2.5cm}{L: #1}
  }
}

\theoremstyle{plain}
\newtheorem{theorem}{Theorem}[section]

\newtheorem{lemma}[theorem]{Lemma}

\theoremstyle{definition}

\newtheorem{remark}[theorem]{Remark}

\newtheorem{example}[theorem]{Example}

\newcommand{\R}{\mathbb{R}}

\newcommand{\N}{\mathbb{N}}

\newcommand{\E}{\mathbb{E}}

\newcommand{\cP}{\mathcal{P}}

\newcommand{\push}[2]{{#1}_\##2}

\renewcommand{\d}{\mathrm{d}}

\renewcommand{\P}{\mathcal P}

\DeclareMathOperator*{\argmin}{\mathrm{argmin}}

\newcommand{\X}{\mathcal{X}}
\newcommand{\Y}{\mathcal{Y}}

\title{The Geometry of Financial Institutions -\\
Wasserstein Clustering of Financial Data}

\makeatletter
\renewcommand{\@oddhead}{\hfill {\fontsize{10}{11}\selectfont \textsl{The Geometry of Financial Institutions -
Wasserstein Clustering of Financial Data}} \hfill \thepage}
\makeatother

\author[L. Riess]{L.\ Riess$^{\dagger,\ddagger}$}
\author[J. Backhoff]{J.\ Backhoff$^{*,\dagger}$}
\author[M. Beiglböck]{M.\ Beiglböck$^{*,\dagger}$}
\author[J. Temme]{J.\ Temme$^{*,\ddagger}$}
\author[A. Wolf]{A.\ Wolf$^{*,\ddagger}$}

\thanks{ This research was funded
in whole or in part by the Austrian Science Fund (FWF) under grants 10.55776/P36835, P35197, and Y782, and the Oesterreichische Nationalbank (OeNB) through project EATE II. For
open access purposes, the author has applied a CC BY public copyright
license to any author accepted manuscript version arising from this
submission. \\
    \begin{tabular}{@{}l@{\,}l}  
        $^*$ &Authors listed in alphabetical order \\  
        $^\dagger$ &University of Vienna \\  
        $^\ddagger$ &Oesterreichische Nationalbank  
    \end{tabular}
}

\date{\today}

\begin{document}

\begin{abstract}

Financial regulation requires the submission of diverse and often highly granular data from financial institutions to regulators. In turn, regulators face the challenge of condensing this data into a comprehensive map that captures the mutual similarity or distance between different institutions and identifies clusters or outliers based on features like size, credit portfolio, or business model.
Additionally, missing data due to varying regulatory requirements for different types of institutions, can further complicate this task.

To address these challenges, we interpret the credit data of financial institutions as probability distributions whose respective distances can be assessed through optimal transport theory. Specifically, we propose a variant of Lloyd's algorithm that applies to probability distributions and uses generalized Wasserstein barycenters to construct a metric space. Our approach provides a solution for the mapping of the banking landscape, enabling regulators to identify clusters of financial institutions and assess their relative similarity or distance.

\end{abstract}

\maketitle
\section{Introduction}

The main contribution of this article is an algorithm that takes several (discrete) probability distributions in a high-dimensional space and clusters them, representing each cluster as a point in a metric space. In this space, the distance between points reflects how different the underlying distributions are. A key feature of the algorithm is its ability to handle missing data — even when entire coordinates are missing systematically from some distributions.

%The contribution of this article is to provide an algorithm which takes as input a number of (discrete) probability distributions in a high dimensional space which are then \emph{clustered and represented} as points in a metric space, where the distance gauges differences between the individual distributions. A crucial feature of the algorithm is that we allow for missing data, which may be missing systematically. Specifically, individual distributions may be missing all information in one or more coordinates.

Our original motivation to devise this type of algorithm stems from a challenge faced by regulators of the financial industry. 
Data that are delivered from financial institutions to the regulator consist of various different formats, from highly aggregate data such as the total volume of the balance sheet, down to very granular data about individual credits described by characteristics such as volume, interest rate, etc. Two individual credits might then be considered as similar if those characteristics take similar numerical values, i.e.\ have small Euclidean distance when viewed as vectors in $\R^d$. To build a distance between financial institutions, one can view these institutions as probability distributions on the space of possible credits (see Section \ref{sec:recon_financial} below) and determine the respective distance as a Wasserstein distance. One can then interpret the landscape of all financial institutions as an ensemble of points in the Wasserstein space, susceptible to familiar methods of clustering, outlier detection, etc. 
A particular challenge which renders the problem more complicated is the ``missingness of data''. Data delivered by institutions may have missing values. More crucially, data may be missing systematically, as different institutions are required to deliver different data at varying levels of detail and granularity. 

The algorithm we propose as a ramification simultaneously clusters probability measures with missing data and represents them as elements of a metric space. The principal structure follows the idea of Lloyd's algorithm for $k$-means clustering and combines it with the concept of generalized Wasserstein barycenters, recently introduced by Delon, Gozlan and Saint-Dizier \cite{JuGoSa21}. A classical approach to dealing with missing data would be (e.g.) to impute from weighted nearest neighbors. However, such a type of imputation systematically skews results since probability distributions with less reported data tend to appear closer to other points which the imputation procedure is attempting to mimic. This type of bias may be undesirable, e.g.\ when one is trying to identify distributions that are ``atypical'', where ``atypical'' could specifically refer to the manner in which data are missing. We devise a particular way of soft imputation which accounts for a random element in filling up missing values, and in particular avoids the above mentioned bias. 

It will be technically convenient to formulate the algorithm in a more general form in Section \ref{sec:clustering_proj_metric}. That is, we cluster and perform soft imputation for points in an arbitrary metric space rather than a Wasserstein space of probability measures. This allows us to simplify the presentation and has the additional benefit that we obtain a version of classical Euclidean $k$-means clustering with missing data. This general perspective allows us in Section \ref{sec:GMM_Rec} to compare our method to existing solutions for reconstructing the metric arrangement of Euclidean points. Additionally, further simulation experiments can be found in the Supplementary Material, and our method is compared to others when viewed solely as a clustering method. In particular, it is compared to $k$-pod \cite{ChChBa16}, another method that tackles the same problem in the Euclidean case, i.e.\ \eqref{eq:NA-kmeans}, meaning the clustering problem without imputing a priori. Notably, our clustering algorithm outperforms $k$-pod consistently. %This general perspective allows us in Section \ref{sec:GMM_Rec} to compare our method with existing solutions for clustering in $\R^d$ with missing data.\LR{Already refer to Appendix here, too??? Say here data is from a central bank???} 
As for the important case of probability distributions with missing data our contribution is detailed in Section \ref{sec:NA_WKMeans} and complemented experimentally in Sections \ref{sec:recon_financial} and \ref{sec:Justifying_NA_W_KMeans} with results using actual (anonymized) loan data reported by financial institutions to Oesterreichische Nationalbank, the central bank of Austria.%\LR{wie schreiben wir das jetzt? Anonymized data from OeNB (Oesterreichische Nationalbank, the central bank of Austria) ?? } 
%\JB{J: Faende ich gut, also so viel Details zu geben wie es uns ermoeglicht ist} \LR{Habs mal sehr kurz gemacht}%stemming from the original motivation of regulation of financial institutions.

% \begin{figure}[ht]
% \centering
% \begin{minipage}{0.85\linewidth}
%     \centering
%     \includegraphics[height=8cm]{landscape.png}
%     %\setlength{\fboxsep}{0pt}\fbox{\includegraphics[height=8cm]{landscape.png}}
% \end{minipage}%
% \caption{Visualization of the Austrian Banking Landscape. \\ Interactive plot: \url{https://lorenzriess.github.io/TGOFI_landscape.html}}
% \label{fig:landscape_labels_intro}
% \end{figure}

\begin{figure}[ht]
\centering
\begin{minipage}{0.5\linewidth}
    \centering
    \includegraphics[height=5.75cm]{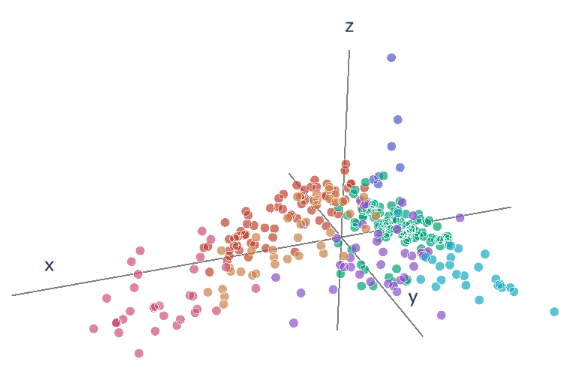}
\end{minipage}%
\begin{minipage}{0.5\linewidth}
    \centering
    \includegraphics[height = 5.75cm]{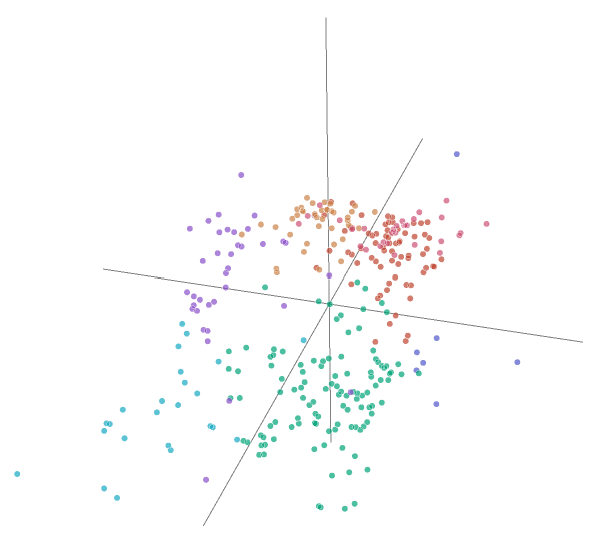}
\end{minipage}
\caption{3-d visualization of the Austrian Banking Landscape from two different perspectives.\\ Interactive plot: \url{https://lorenzriess.github.io/TGOFI_landscape.html}}
\label{fig:landscape_labels}
\end{figure}

\subsection*{Related Literature} 

Clustering distributions using Wasserstein distances has been explored in various works including \cite{PaDoDrKoTsYa21,ZhChYa22,VaWa19,HuHoDaNgYuBuPh21,HoNgYuBuHuPh17}. This approach has also been applied in financial contexts, such as in \cite{HoIsMu21}, for the clustering of market regimes. In \cite{StJe17}, Staib et al.\ proposed Wasserstein $k$-means++ as well as an initialization strategy, generalizing the classical $k$-means++ algorithm, cf.\ \cite{ArVa07}. However, to the best of our knowledge, clustering distributions with missing coordinates has not been investigated before. To address this problem, we rely heavily on the concept of the generalized Wasserstein barycenter, introduced by Delon, Gozlan, and Saint-Dizier in \cite{JuGoSa21}, which extends the classical Wasserstein barycenter of Agueh and Carlier \cite{AgCa11}. In the Euclidean case, we are aware of one method for the $k$-means problem with missing values which does not impute points beforehand. This method is called $k$-pod and was introduced by Chi et.al.\ in \cite{ChChBa16}. For computational aspects related to optimal transport and regularized Wasserstein distance, we refer to the work of Cuturi \cite{Cu13} and the book \cite{Computational_OT} by Cuturi and Peyr\'e. Additionally, we use the \texttt{POT} package (Python package for optimal transport, see \cite{FlCoGrAlBoChChCoFaFoGaGaJaRaReRoScSeSuTaToVa21}) extensively for implementation purposes.

\section{\texorpdfstring{$k$}k-means Clustering in Metric Spaces: A Summary}\label{sec:KMeans_abstract}
Let $x_1,\ldots,x_N$ be given points in a fixed metric space $(\X,d)$. The metric $k$-means problem is concerned with assigning the points to $k$ clusters. The clusters hereby are governed by $k$ barycenters, i.e.\ an appropriate notion of average of the points in the corresponding cluster. $k$-means was first introduced by MacQueen in \cite{MacQ67}. The problem can be formalized as
\begin{align}\tag{KM}\label{general_KMeans} 
\min\limits_{\substack{c_j\in\X\\ a\in[k]^N}}\sum_{i=1}^N d(x_i,c_{a_i})^2,
\end{align}
where we use the notation $[n]:=\{1,\ldots,n\}$, $n\in\mathbb{N}$.
In this formulation $a$ denotes a vector of assignments, i.e.\ $a_i$ indicates the cluster membership of data point $x_i$. The points $c_1,\ldots,c_k$ serve as cluster barycenters. Furthermore, a notion of barycenter, typically a function of some of the data points, is needed. Then, problem \eqref{general_KMeans} can be tackled by the well-established Lloyd algorithm, cf.\ \cite{LLoyd82}.

Following an initialization phase, two steps are iterated:
\begin{enumerate}[i), nosep]
\item assignment step: given barycenters $c_j\in\mathcal{X}$, for each $i\in[N]$ pick 
\begin{align}
    a_i\in\argmin\limits_{j\in [k]}d(x_i,c_j),
\end{align}
%\begin{align} 
%a_i\in\argmin\lbrace d(x_i,c_j):j\in[k]\rbrace,
%\end{align}
\item barycenter step: given assignment $a$, update the barycenters, i.e.\  for each $j\in[k]$ pick
\begin{align}\label{eq:centroid_general_kmeans} 
c_j\in\text{barycenter}(\{x_i : a_i = j\}).
\end{align}
\end{enumerate} 

Let us give two examples with particular choices of data space $\mathcal{X}$, metric $d$, and barycenter operation, which will be of interest in the sequel:

\begin{example}\label{example_KMeans_Euclidean}
Let $\mathcal{X}=\R^d$ and $d(x,y)=\lVert x-y\rVert_2$, the Euclidean space and distance. Then \eqref{general_KMeans} simplifies to the classical $k$-means problem in Euclidean space,
\begin{align*} 
\min\limits_{\substack{c_j\in\R^d\\ a\in[k]^N}}\sum_{i=1}^N\lVert x_i-c_{a_i}\rVert_2^2.
\end{align*}
The barycenter function is the Euclidean mean/average.
\end{example}

\begin{example}[Wasserstein $k$-means]\label{example_Wasserstein_KMeans}
Here $\mathcal{X} = \mathcal{P}_2(\R^d)$, the space of probability measures on $\R^d$ with finite second moments, and $d$ is taken to be the Wasserstein distance $ W_2$ (recalled in \eqref{eq:Wasserstein_dist} below). We will henceforth write $\mu_i=x_i$ for $i\in[N]$. Then \eqref{general_KMeans} turns into 
\begin{align*} 
\min\limits_{\substack{\nu_j\in\mathcal{P}_2(\R^d)\\ a\in[k]^N}}\sum_{i=1}^NW_2^2(\mu_i,\nu_{a_i}).
\end{align*}
As barycenter one takes the Wasserstein barycenter, i.e.\ the barycenter updating step \eqref{eq:centroid_general_kmeans} reads 
\begin{align} 
\nu_j\in\argmin\limits_{\nu\in\mathcal{P}_2(\R^d)}\sum_{i:a_i = j}W_2^2(\mu_i,\nu).
\end{align}
\end{example}
Subsequently, we want to generalize the previous examples, in particular Example \ref{example_Wasserstein_KMeans}, to the case in which certain coordinates/marginals are missing for some observations, for an illustration see \Cref{fig:6measures_clustering} in Section \ref{sec:NA_WKMeans}.

When applying Lloyd's algorithm to a specific setup, one needs to specify a distance function and a notion of barycenter. In the case of a metric space it is natural to take a Fr\'echet mean for the barycenter operation, see \cite{Fr48}. In the case of missing coordinates, however, the data does not come from one metric space but from several different spaces, and hence several distances need to be considered simultaneously. In particular, it is also necessary to consider a generalized notion of barycenter. 
We will introduce an approach to this challenge in the next section and discuss our specific examples in more detail in Example \ref{ex:NA_KMeans} and Section \ref{sec:NA_WKMeans}, respectively. 

\section{Clustering Projected Elements of Metric Spaces}\label{sec:clustering_proj_metric}

Suppose we are working in a metric space $(\X,d)$ and have points $x_1,\ldots,x_N\in\X$ that we want to assign to $k$ clusters. However, what we actually observe are the points $\tilde{x}_i:=\varphi_i(x_i)$, where $\varphi_i:\X\to\X_i$ is a known map into another metric space $\X_i$. For instance, in the case $\X=\R^d$ the functions $\varphi_i$ might be projections onto subspaces. Having observed $\tilde{x}_i\in\X_i$ define for $y\in \X$
\begin{align*} 
d_i(\tilde{x}_i,y) &:= d(\varphi_i^{-1}(\tilde{x}_i),y)\\&:=\inf\limits_{x\in\X:\varphi_i(x)=\tilde{x}_i} d(x,y),
\end{align*}
which serves as a type of dissimilarity measure of a point $y\in\X$ in the ``full'' metric space to the observed point $\tilde{x}_i$. 

\subsection{Problem and Algorithm}\label{sec:problem_algo_general_metric}
Based on these notations, we introduce the problem
\begin{align}\label{general_KMeans_onmapped_points} 
\min\limits_{\substack{c_j\in\X\\ a\in[k]^N}}\sum_{i=1}^N d_i(\tilde{x}_i,c_{a_i})^2.
\end{align}
I.e.\ we want to find optimal cluster barycenters $c_1,\ldots,c_k\in\X$, as well as a vector $a$ that optimally assigns each observed point to a cluster. It is important to note that we are looking for cluster barycenters in the ``full'' space $\X$. This ensures that we are be able to compare them to each observed point $\tilde{x}_i$ using $d_i(\tilde{x}_i,\cdot)$.

We propose tackling \eqref{general_KMeans_onmapped_points} using the following two steps which are iterated in a Lloyd algorithm fashion after initializing the barycenters:
\begin{enumerate}[i), nosep]
\item \label{general_label_update}assignment step: given barycenters $c_j\in\X$, set
\begin{align}
    a_i \in \argmin\limits_{j\in[k]}d_i(\tilde{x}_i,c_j),
\end{align}
\item \label{general_bary_update}barycenter step: given assignment $a$, update the barycenters, i.e.\ for each $j\in[k]$, set 
\begin{align}\label{eq:general_bary_update}
    c_j\in\argmin\limits_{y\in\X}\sum\limits_{i:a_i=j} d_i(\tilde{x}_i,y)^2.
\end{align}

\end{enumerate}

Let us note that step \ref{general_bary_update} does not necessarily admit minimizers. However, in our two applications --\ NA $k$-means (cf.\ Example \ref{ex:NA_KMeans}) and NA Wasserstein $k$-means (cf.\ Section \ref{sec:NA_WKMeans}) -- minimizers always exist. A more precise discussion of this technical point is given in Appendix \ref{sec:ex_barys}. We use the abbreviation ``NA'' for ``Not Available''. Thus, NA $k$-means and NA Wasserstein $k$-means refer to the respective algorithms of Examples \ref{example_KMeans_Euclidean} and \ref{example_Wasserstein_KMeans} adapted to handle missing coordinates. 

Concerning initialization, we (slightly) adapt the widely used $k$-means++ initialization algorithm introduced in \cite{ArVa07}. In its original form, the initial cluster barycenters are selected from the observed points themselves. The algorithm begins by choosing the first barycenter uniformly at random from the observed points and then repeatedly choosing points at random with probability proportional to their squared distance from the already chosen barycenters until $k$ barycenters are chosen. This approach is not directly applicable to our setting because not all points are fully observed, and thus not all pairwise distances can be computed. %Nevertheless, we propose applying the $k$-means++ initialization algorithm to the fully observed points only, i.e.\ those points $x_i$ with $\varphi_i = \mathrm{Id}_\X$.
Therefore we apply the $k$-means++ initialization algorithm to the subset of fully observed points only, i.e.\ those points $x_i$ with $\varphi_i = \mathrm{Id}_\X$.

\subsection{Using Clustering for Imputation}\label{sec:imputing_general}

After clustering we obtain a vector of assignments $a$ and barycenters $c_1,\ldots,c_k$. We can use these to impute the not fully observed points. For this sake define
\begin{align}\label{def_of_index_sets_missing_full} 
I_f:=\{i:\varphi_i = \mathrm{Id}_\X\},\quad I_m:=[N]\setminus I_f,
\end{align}
i.e.\ respectively the set of indices of points being observed in $\X$ (i.e.\ fully observed data) and the set of indices of points that are only observed after some non-trivial map (i.e.\ with missing data).
The final clusters are defined by letting for $j\in[k]$ 
\begin{align*} 
C_j := \{i\in[N]:a_i = j\}.
\end{align*}

We want to use the clusters to find, for a point $\tilde{x}_i$ with $i\in I_m$, a probability measure on $\X$, i.e.\ an element of $\mathcal{P}(\X)$, which attempts to concentrate around the true (unobserved) point $x_i$. We define for $i\in I_m$ the set of indices we use for filling up the missing information of $\tilde{x}_i$, as
\begin{align*} 
J_i:=I_f\cap C_{a_i}.
\end{align*}
This says that we want to use the ``full'' points in the same cluster as $\tilde{x}_i$ in order to compensate the incomplete information that we have about $\tilde{x}_i$. It can of course happen that $J_i=\emptyset$. In this case we use the corresponding cluster barycenter $c_{a_i}$ to complement the information about $\tilde{x}_i$. In the following we assume $J_i\neq\emptyset$.
Taking now some $x_\ell$ with $\ell\in J_i$ we set
\begin{align} \label{eq:fill_up_point}
y^i_\ell\in \argmin\limits_{\substack{x\in\X\\ \varphi_i(x)=\tilde{x}_i}} d(x,x_\ell),
\end{align}
which will be one potential choice of ``filling up''. We introduce the shorthand $D^i_\ell:=d(y^i_\ell,x_\ell)=d_i(\tilde{x}_i,x_\ell)$. In order to determine the weights of a probability measure, we set $p^i_\ell:=f(D^i_\ell)$ with $f:[0,\infty)\to [0,\infty)$ being some positive decreasing function fixed in advance, e.g.\ $f(x):= \exp(-x^2)$. This way $p^i_\ell\geq 0$ and also $\sum_{\ell\in J_i}p^i_\ell=1$ after possibly renormalizing the weights by a positive constant. Having obtained the weights, we define the probability measure which should represent a \emph{randomly reconstructed} $x_i$ as
\begin{align}\label{eq:imputed_as_measure_abstract_metric_space} 
\theta_i:=\sum_{\ell\in J_i}p^i_\ell\delta_{y^i_\ell}.
\end{align}
To also embed the fully observed points, for $i\in I_f$ we set $\theta_i:=\delta_{x_i}$, that is, the Dirac delta concentrated on $x_i$. For $i\in I_m $ with $J_i=\emptyset$, we use the corresponding cluster barycenter $c_{a_i}$ and set $\theta_i $ to be the Dirac delta concentrated on some minimizer of $d(\varphi_i^{-1}(\tilde{x}_i),c_{a_i})$. Thus, we embed all observed points $\tilde{x}_1,\ldots,\tilde{x}_N$ in the same space $\mathcal{P}(\X)$. In other words, we identify the possibly unobserved point $x_i$ with a probability measure $\theta_i$. %By this procedure we have embedded all observed points $\tilde{x}_1,\ldots,\tilde{x}_N$ in the same space, i.e.\ $\mathcal{P}(\X)$. In other words, we have identified the possibly unobserved point $x_i$ with a probability measure $\theta_i$.

In order to compare the hitherto constructed probability measures, as we will need to do in Section \ref{sec:exp_results}, we define a (generalized) metric $\rho$ on $\mathcal{P}(\X)$. Readers primarily interested in clustering and imputation may skip this construction. The (generalized) metric $\rho$ is defined via the cost induced by the \emph{product} or \emph{independent coupling}. %, we think of averaging out all the possible ways to impute $\tilde{x}_i$. 
 That is, for $\mu,\nu\in\mathcal{P}(\X)$ we set $\rho(\mu,\nu):=0$ if $\mu=\nu$ and otherwise 
\begin{gather}
\begin{aligned}\label{prod_dist_general} 
\rho(\mu,\nu) :=\int_{\X}\int_{\X}d(x,x')\,\mu(\mathrm{d}x)\,\nu(\mathrm{d}x').
\end{aligned}
\end{gather}
For completeness we provide a short lemma proving that $\rho$ is indeed a (generalized) metric.

\begin{lemma}\label{lem:rho_gen_metric}
Let $(\X,d)$ be a metric space and define on $\mathcal{P}(\X)$ the map $\rho:\mathcal{P}(\X)\times\mathcal{P}(\X)\to [0,\infty]$ by
\begin{align*} 
\rho(\mu,\nu):=
\begin{cases} 
\int_{\X\times\X}d(x,x')\,(\mu\otimes\nu)(\mathrm{d}x,\mathrm{d}x'), & \text{if } \mu\neq \nu\\ 
0, & \text{if } \mu=\nu.
\end{cases}
\end{align*}
Then, $\rho$ is a generalized metric on $\mathcal{P}(\X)$.
\end{lemma}
\begin{proof}
Symmetry of the generalized metric is clear due to the symmetry of the underlying distance $d$. Concerning the triangle inequality take $\mu,\nu,\theta\in\mathcal{P}(\X)$ which we assume to be different from each other (otherwise there is nothing to prove). Furthermore, take three independent random variables $X\sim\mu,Y\sim\nu,Z\sim\theta$. Then,
\begin{align*} 
\rho(\mu,\theta) &=   \int_{\X\times\X}d(x,x')\,(\mu\otimes\theta)(\mathrm{d}x,\mathrm{d}x')\\
&= \E[d(X,Z)] \\
&\leq    \E[d(X,Y)+d(Y,Z)] = \rho(\mu,\nu) + \rho(\nu,\theta),
\end{align*}
which proves the triangle inequality for $\rho$.
Concerning definiteness, if $\mu=\nu$, we have $\rho(\mu,\nu)=0$ by definition. 

Suppose now that
\begin{equation*}
    \int_{\X\times\X}d(x,x')\,(\mu\otimes\nu)(\mathrm{d}x,\mathrm{d}x') = 0.
\end{equation*}
This implies $d(x,x')=0$ for $\mu\otimes\nu$-almost all $(x,x')$. Since $d$ is a metric, we have $x=x'$, $\mu\otimes\nu$-almost surely. Thus, $\mu=\delta_x=\nu$ for some $x\in \X$.
\end{proof}
% pulled this out to save some space: := \int_{\X\times\X}d(x,x')\mathrm{d}\mu\otimes\nu(x,x')\\  
Using a Wasserstein distance or a similar notion of distance in \eqref{prod_dist_general}, imputed points would be biased to be closer than ``fully'' observed points. The metric $\rho$ is designed to avoid this type of bias. Indeed, our choice in \eqref{prod_dist_general} formalises the idea that missing values are not imputed in a deterministic sense but in a \emph{random} or \emph{soft} fashion, as indicated in the introduction. The distance between two randomly imputed values is then estimated as an \emph{independent} average of distances, following the rationale that the imputation for one point does not inform the imputation for a different point. 

\begin{example}[NA $k$-means]\label{ex:NA_KMeans}
Corresponding to classical $k$-means, i.e.\ Example \ref{example_KMeans_Euclidean}, consider $\X = \R^d$. Let $\varphi_i = P_i :\R^d\to\R^{d_i}$ be projections onto some of the coordinates. Then, $\tilde{x}_i=P_i(x_i)$ and \eqref{general_KMeans_onmapped_points} becomes
\begin{align}\label{eq:NA-kmeans} 
\min\limits_{\substack{c_j\in\R^d\\a\in[k]^N}}\sum_{i=1}^N \lVert \tilde{x}_i-P_i(c_{a_i})\rVert_2^2.
\end{align}
The cluster assignment iteration step assigns each point to the cluster whose barycenter is closest, considering only the known coordinates.
The barycenter updating step is solved by
\begin{equation*}
    \Big(\sum_{i:a_i=j}P_i^TP_i\Big)^{-1}\sum_{i:a_i=j}P_i^TP_i(x_i),
\end{equation*} 
which, in each coordinate, corresponds to the average of all the points in the cluster for which that coordinate is available. We detail in Section \ref{sec:NA_WKMeans} how to deal with the case where the inverse does not exit. By imputing missing values as described above, we obtain measures in $\mathcal{P}_2(\R^d)$, of which we can then calculate pairwise distances using the metric $\rho$, defined in \eqref{prod_dist_general}. Experiments using this method may be found in Section \ref{sec:GMM_Rec}.

It is worth noting that \cite{ChChBa16} considers the same loss function, i.e.\ \eqref{eq:NA-kmeans}, but proposes a different algorithm for clustering Euclidean points with missing values, known as $k$-pod. In the Supplementary Material a comparison to their algorithm can be found when we view our method solely as clustering algorithm. Notably, our proposed method, NA $k$-means, consistently outperforms $k$-pod.
\end{example}

Before discussing our second example in Section \ref{sec:NA_WKMeans}, which generalizes \Cref{example_Wasserstein_KMeans}, we first recall the necessary notions from optimal transport theory.

\section{Wasserstein Distance and Generalized Wasserstein Barycenter: a Summary}
\subsection*{Optimal Transport and Wasserstein Distance}
Let $ \mu, \nu$ be probability measures on Polish spaces $X,Y$ respectively. For a measurable map $T:X\to Y$ we use $_\#$ to denote the push-forward operator of measures (image measure), i.e.\ for a measurable set $B\subset Y,$ we put $\push{T}{\mu}(B):=\mu(T^{-1}(B))$.
Denote by
\begin{align*} 
\Pi(\mu,\nu):=\{\pi\in\mathcal{P}(X\times Y) : \push{\mathrm{proj}_X}{\pi}=\mu,\push{\mathrm{proj}_Y}{\pi}=\nu\}
\end{align*}
the set of couplings of $\mu$ and $\nu$, i.e.\ the set of all measures on the product space having $\mu$ and $\nu$ as marginals. 
The Kantorovich problem for a cost function $c:X\times Y\to\R_+$, introduced in \cite{Ka42}, is
\begin{align*} 
\inf_{\pi\in\Pi(\mu,\nu)}\int_{X\times Y} c(x,y) \,\pi(\mathrm{d}x,\mathrm{d}y). \tag{KP}
\end{align*}
We specialize to the case $X=Y=\R^d$ and $c(x,y)=\lVert x-y\rVert_p^p$ for $p\in [1,\infty)$. The Wasserstein distance $W_p$ on the space $\mathcal{P}_p(\R^d)$ of probability measures with finite $p$-moment is then defined via
\begin{align}\label{eq:Wasserstein_dist} 
W_p(\mu,\nu)^p:=\inf\limits_{\pi\in\Pi(\mu,\nu)}\int_{\R^d\times\R^d}\lVert x-y\rVert_p^p\ \pi(\mathrm{d}x,\mathrm{d}y).
\end{align}

\subsection*{Wasserstein Barycenter}
A notion of averaging probability measures that has recently received significant attention is the concept of Wasserstein barycenters, introduced by Agueh and Carlier \cite{AgCa11}. %Uniqueness and existence of the Wasserstein barycenter are studied in there as well. 
A Wasserstein barycenter of $\mu_1,\ldots,\mu_n\in\mathcal{P}_2(\R^d)$ with weights $\lambda_1,\ldots,\lambda_n\geq  0$ summing to $1$, is a solution of
\begin{align*} 
\inf\limits_{\nu\in\mathcal{P}_2(\R^d)}\sum_{i=1}^n\lambda_iW_2^2(\mu_i,\nu).
\end{align*}
Since we want to generalize the algorithm introduced in Sections \ref{sec:KMeans_abstract}-\ref{sec:clustering_proj_metric} to the setting of probability measures, we require a variant of the Wasserstein barycenter that is still applicable when only some coordinates of the measures are known. A suitable concept, the \emph{generalized Wasserstein barycenter}, was introduced by Julien, Gozlan and Saint-Dizier in \cite{JuGoSa21}. 
To formally define it, let probability measures $\mu_1,\ldots,\mu_n\in\mathcal{P}_2(\R^d)$ be given along with linear maps $P_i:\R^d\to\R^{d_i}$. In our intended applications, these maps correspond to projections onto some of the coordinates. A generalized Wasserstein barycenter of the push-forwarded measures $\push{P_1}{\mu_1},\ldots,\push{P_n}{\mu_n}$ with associated weights $\lambda_1,\ldots,\lambda_n\geq 0$ summing to $1$, is a solution of
\begin{align}\label{Generalized_Wasserstein_Barycenter} 
\inf\limits_{\nu\in\mathcal{P}_2(\R^d)}\sum_{i=1}^n \lambda_iW_2^2(\push{P_i}{\mu_i},\push{P_i}{\nu}).
\end{align}
To solve the problem, it is useful to reformulate it as a classical Wasserstein barycenter problem. By associating each projection $P_i$ with a matrix $P_i\in\R^{d_i\times d}$, we define $A:=\sum_{i=1}^n \lambda_iP_i^TP_i$
and assume in the following that $A$ is invertible. Next, we set $\bar{\mu}_i:=\push{(A^{-1/2}P_i^T)}{\mu_i}\in\mathcal{P}_2(\R^d)$ for $i\in[n]$, and consider the classical Wasserstein barycenter problem
\begin{align}\label{classical_bary_problem} 
\inf\limits_{\bar{\nu}\in\mathcal{P}_2(\R^d)}\sum_{i=1}^n\lambda_iW_2^2(\bar{\mu}_i,\bar{\nu}).
\end{align}
Then, $\nu$ is a solution to \eqref{Generalized_Wasserstein_Barycenter} if and only if $\bar{\mu}=\push{A^{1/2}}{\nu}$ is a solution to \eqref{classical_bary_problem}, see Proposition 3.1 in \cite{JuGoSa21}. Regarding computational aspects, especially in the discrete case, efficient algorithms for computing Wasserstein barycenters have already been developed; see e.g.\ \cite{CuDo14}, as well as \cite{AlBaCuMa16}, and are implemented in the Python package \texttt{POT} (see \cite{FlCoGrAlBoChChCoFaFoGaGaJaRaReRoScSeSuTaToVa21}).

\section{NA WASSERSTEIN \texorpdfstring{$k$}k-MEANS}\label{sec:NA_WKMeans}
\subsection*{Algorithm Description}
We can now discuss the case of $\X = \P_2(\R^d)$ in detail, as we have established the two necessary operations — the Wasserstein distance and the generalized Wasserstein barycenter — for an algorithm as described in \Cref{sec:clustering_proj_metric}.

Suppose we want to cluster probability measures $\mu_1,\ldots,\mu_N\in\mathcal{P}_2(\R^d)$ into $k$ clusters but we only observe their push-forwards $\push{P_i}{\mu_i}\in\mathcal{P}_2(\R^{d_i})$ under projections $P_i:\R^d\to\R^{d_i}$ for $i\in [N]$. Thus, in the language of \Cref{sec:clustering_proj_metric}, we have $\varphi_i(\cdot): = \push{P_i}{(\cdot)}$, and $\tilde{x}_i=\push{P_i}{\mu_i}=:\tilde{\mu}_i$. Problem \eqref{general_KMeans_onmapped_points} then reads as 
\begin{align*} 
\inf\limits_{\substack{\nu_j\in\mathcal{P}_2(\R^d)\\a\in[k]^N}}\sum_{i=1}^NW_2^2(\tilde{\mu}_i,\push{P_i}{\nu_{a_i}}).
\end{align*}
We can tackle this problem by the iterations suggested in Section \ref{sec:problem_algo_general_metric}, which read as
\begin{enumerate}[i), nosep]
\item update cluster assignments given barycenters $\nu_j$, i.e.\ for each $i\in[N]$ set
\begin{align}\label{NA_WKMeans_assignment steps} 
a_i\in\argmin\limits_{j\in[k]}W_2(\tilde{\mu}_i,\push{P_i}{\nu_j}),
\end{align}
\item \label{NAWKMeans_bary_update_step} update cluster barycenters given an assignment $a$, i.e.\ for each $j\in[k]$ set
\begin{align*} 
\nu_j\in\argmin\limits_{\nu\in\mathcal{P}_2(\R^d)}\sum\limits_{i:a_i=j}W_2^2(\tilde{\mu}_i,\push{P_i}{\nu}).
\end{align*}
\end{enumerate}
Again, we aim to find an optimal assignment vector $a$ that assigns each measure to one of the $k$ clusters, as well as optimal barycenters, specifically generalized Wasserstein barycenters, $\nu_1,\ldots,\nu_k$. Note that once again, we are looking for barycenters in the full space $\P_2(\R^d)$, i.e.\ having all coordinates.

It may happen that a cluster consists entirely of measures missing the same coordinate, i.e.\ for some $j\in[k]$, we have 
\begin{equation*}
    \bigcap_{i:a_i=j} \ker P_i \neq \{0\}.
\end{equation*}
To avoid numerical problems in such cases, one can initialize the barycenters without missing coordinates and adapt step \ref{NAWKMeans_bary_update_step} to include the previous barycenter with a small weight. This ensures that all subsequent barycenters have no missing coordinates. Specifically, assume that at iteration $t$ the barycenters are $\nu^{(t)}_j\in\mathcal{P}_2(\R^d)$ and a new assignment vector $a^{(t+1)}\in[k]^N$ has been computed. Given a weight $\lambda^{(t)}\in (0,1)$, set 
\begin{align} \label{eq:bary_update_with_weight}
\nu_j^{(t+1)}\in\argmin\limits_{\nu\in\mathcal{P}_2(\R^d)}(1-\lambda^{(t)})\sum\limits_{i:a_i^{(t+1)}=j} W_2^2(\tilde{\mu}_i,\push{P_i}{\nu}) + \lambda^{(t)}W_2^2(\nu_j^{(t)},\nu).
\end{align}
%\begin{gather}
%\begin{aligned} \label{eq:bary_update_with_weight}
%\nu_j^{(t+1)}\in&\argmin\big\lbrace (1-\lambda^{(t)})\sum\limits_{i:a_i^{(t+1)}=j} W_2^2(\tilde{\mu}_i,P_i\#\nu) \\&+ \lambda^{(t)}W_2^2(\nu_j^{(t)},\nu):\nu\in\mathcal{P}_2(\R^d)\big\rbrace.
%\end{aligned}
%\end{gather}
Experimentally we have observed that $\lambda^{(t)}:=\frac{1}{(t+1)^{1/2}}$ works well. Having defined the two iterative steps, we now present the main computational contribution of this work: Algorithm \ref{alg:NA_WKMeans}. If no prior information is available, we propose initializing it as described in \Cref{sec:problem_algo_general_metric}.

\begin{algorithm}[ht]
   \caption{NA Wasserstein $k$-means}\label{alg:NA_WKMeans}
\begin{algorithmic}
   \STATE {\bfseries Input:} $N$ observed measures $\tilde{\mu}_i=\push{P_i}{\mu_i}$, number of clusters $k$, maximum number of iterations $T$, weighting schedule $(\lambda^{(t)})_{t=0}^T$
   \STATE {\bfseries Result:} barycenters $\nu_j$, assignments $a$
   \STATE Initialize $t=0,~\nu_1^{(0)},\ldots,\nu_k^{(0)}\in\mathcal{P}_2(\R^d)$, and $a^{(0)}\in[k]^N$
   \WHILE{$t<T$}
   \FOR{$i=1$ {\bfseries to} $N$}
   \STATE \vspace*{-6mm}\begin{gather}\begin{aligned}\label{eq_algo_assignment} 	a_i^{(t+1)}:=\argmin\limits_{j\in[k]}W_2^2(\tilde{\mu}_i,\push{P_i}{\nu_j^{(t)}})\end{aligned}\end{gather}\vspace*{-7mm}
   \ENDFOR
   \FOR{$j=1$ {\bfseries to} $K$}
   \STATE choose $\nu_j^{(t+1)}$ according to \eqref{eq:bary_update_with_weight} by including the old barycenter $\nu_j^{(t)}$ with weight $\lambda^{(t)}$
   \ENDFOR
   \IF{$a^{(t)} == a^{(t+1)}$}
   \STATE {\bfseries break}
   \ENDIF
   \STATE $t:=t+1$
   \ENDWHILE
\end{algorithmic}
\end{algorithm}

In \eqref{eq_algo_assignment} we use the entry of the old assignment vector $a_i^{(t)}$ if it is a minimizer.
This leads to a simple convergence result. For this sake let
\begin{align}\label{eq:Loss_NA_Wasserstein_kmeans} 
L(\nu_1,\ldots,\nu_K,a):=\sum_{i=1}^NW_2^2(\push{P_i}{\mu_i},\push{P_i}{\nu_{a_i}}).
\end{align}
\begin{theorem}\label{thm_algo_convergence}
Given an initialization $\nu_1^{(0)},\ldots,\nu_k^{(0)}\in\mathcal{P}_2(\R^d),~a^{(0)}\in[k]^N$, Algorithm \ref{alg:NA_WKMeans} strictly decreases $L$ until it terminates after finitely many steps, for any choice of $(\lambda^{(t)})_{t\in\N}$ (even for $T=\infty$).
If in every iteration each cluster has a measure with all coordinates, then Algorithm \ref{alg:NA_WKMeans} with $\lambda^{(t)}=0$ yields a local minimum of $L$.
\end{theorem}

\begin{remark}[Speed of convergence]
    It is known that the classical Euclidean $k$-means algorithm can require an exponential number of iterations in the worst-case. Specifically, there exists a lower bound $2^{\Omega(N)}$ even in two dimensions, see \cite{Va11, DaSe06}. Since Algorithm \ref{alg:NA_WKMeans} includes $k$-means as a special case (specifically, when applied to Dirac measures, i.e.\ Euclidean data, without missing values), this worst-case behavior also applies to our method. It cannot be worse than exponential, however, due to the trivial upper bound $k^N$ of possible cluster assignments.

    Nevertheless, in practice, $k$-means typically converges within a moderate number of iterations - often around $20$ to $50$ - when clustering a not-too-large number of objects, see \cite{HaSa05, BrGaJoVa14}. In our experiments, we observed a similar behavior for both NA Wasserstein $k$-means and NA $k$-means. On average, the number of iterations hovered around $30$ iterations without exceeding $50$. Specifically, in all experiments, we set the maximum number of iterations to $100$ and this bound was never reached.
\end{remark}

\begin{proof}[Proof of Theorem \ref{thm_algo_convergence}.] 
Let us first show that the loss function decreases monotonically. Using the assignment step in the first and the barycenter updating step in the second inequality, we obtain
\begin{align*} 
(1-\lambda^{(t+1)})L(\nu_1^{(t)},\ldots,\nu_k^{(t)},a^{(t)}) 
=(1-\lambda^{(t+1)}) \sum_{i=1}^NW_2^2\big(\push{P_i}{\mu_i},\push{P_i}{\nu_{a_i^{(t)}}^{(t)}}\big) \\ 
\overset{(\star)}{\geq} (1-\lambda^{(t+1)})\sum_{i=1}^NW_2^2\big(\push{P_i}{\mu_i},\push{P_i}{\nu^{(t)}_{a_i^{(t+1)}}}\big) \\ 
= \sum_{j=1}^k \Big[ (1-\lambda^{(t+1)})\sum_{i:a_i^{(t+1)}=j}W_2^2(\push{P_i}{\mu_i},\push{P_i}{\nu_j^{(t)})}+\lambda^{(t+1)}\underbrace{W_2^2(\nu_j^{(t)},\nu_j^{(t)})}_{=0}\Big] \\ 
\geq \sum_{j=1}^k \Big[ (1-\lambda^{(t+1)})\sum_{i:a_i^{(t+1)}=j}W_2^2(\push{P_i}{\mu_i},\push{P_i}{\nu_j^{(t+1)}}) +\lambda^{(t+1)}\underbrace{W_2^2(\nu_j^{(t)},\nu_j^{(t+1)})}_{\geq 0}\Big] \\ 
\geq (1-\lambda^{(t+1)})\sum_{i=1}^NW_2^2\big(\push{P_i}{\mu_i},\push{P_i}{\nu_{a_i^{(t+1)}}^{(t+1)}}\big)\\ 
 = (1-\lambda^{(t+1)})L(\nu_1^{(t+1)},\ldots,\nu_k^{(t+1)},a^{(t+1)}).
\end{align*}
Since $1-\lambda^{(t+1)} > 0$, this shows the monotonicity. In inequality $(\star)$, equality holds if and only if there is no change in the assignment step, i.e. if
\begin{align*} 
a_i^{(t)} = a_i^{(t+1)},~\forall i\in[N].
\end{align*}
In this case the algorithm terminates. If there is some $i\in[N]$ such that $a_i^{(t)}\neq a_i^{(t+1)}$, then we have a strict inequality, since the assignment in Algorithm \ref{alg:NA_WKMeans}, cf.\ \eqref{eq_algo_assignment}, is only updated when $L$ decreases due to the new assignment.

Since only a finite number of possible assignments exist, due to the fact that we are clustering finitely many measures into finitely many clusters, the algorithm strictly decreases $L$ and terminates after a finite number of steps, thus concluding the proof.
\end{proof}

\begin{remark}[Convergence in abstract metric spaces]
    In the setting of abstract metric spaces, that is, Section \ref{sec:problem_algo_general_metric}, we can modify the assignment step as in Algorithm \ref{alg:NA_WKMeans} -- that is, assignments are only updated when the loss function strictly decreases. If we further assume that barycenters always exist, the convergence of the algorithm follows by the same arguments as in the previous proof.
\end{remark}

\subsection*{A Toy Example}

To illustrate the algorithm, consider six measures on the plane, one of which misses the vertical coordinate (indicated by vertical lines). Each measure consists of three support points with equal weight. The measures are visualized by different colors and depicted on the left of Figure \ref{fig:6measures_clustering}.
When clustering these measures into three clusters, the natural choice of clusters is evident. Applying the above algorithm indeed yields the results shown in Figure \ref{fig:6measures_clustering}.
To obtain these results, we calculated free support barycenters at each step of the algorithm with a fixed support size of three (cf.\ \cite{AlEs16,CuDo14}). For the blue and brown barycenters, each support point is precisely the average of the support points of the two measures in that cluster. For the pink barycenter the situation is slightly different. In the $y$ coordinate it inherits the values from the red measure, while as in the $x$ coordinate the support points are the averages of the $x$ coordinates of a red and a rose support point, respectively.
\begin{figure*}[t]
\centering
\begin{minipage}{.495\textwidth}
\centering
\includegraphics[width=.995\linewidth]{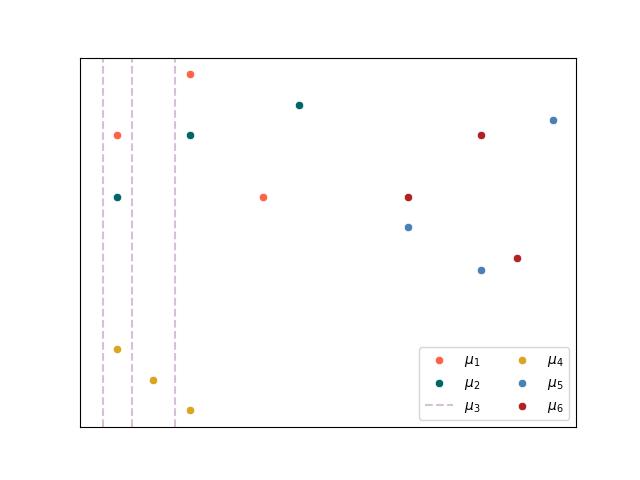}
\end{minipage}%
\begin{minipage}{.495\textwidth}
\centering
\includegraphics[width=.995\linewidth]{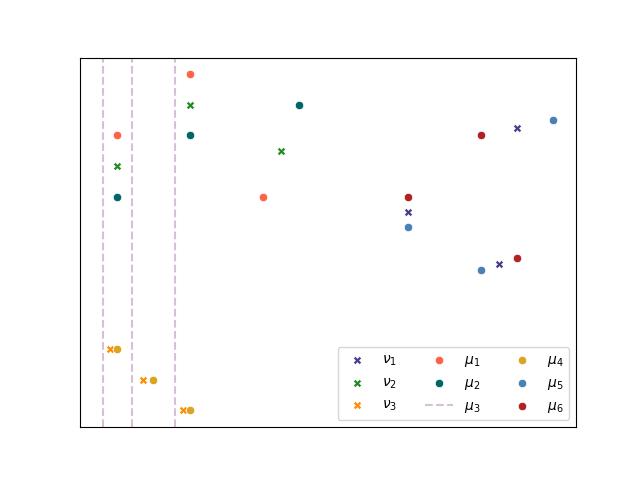}
\end{minipage}
\caption{(left): $6$ Measures on the Plane of Which One Misses the Vertical Coordinate of Its Points. (right) Clustering of the $6$ Measures Into $3$ Clusters.}\label{fig:6measures_clustering}
\end{figure*}

\subsection*{Imputation via NA Wasserstein \texorpdfstring{$k$}k-means}
Next, we consider how the imputation procedure described in Section \ref{sec:imputing_general} looks, when $\X = \mathcal{P}_2(\R^d)$. Having clustered measures $\push{P_i}{\mu_i}\in\mathcal{P}_2(\R^{d_i}),~i\in[N]$, we obtain generalized Wasserstein barycenters $\nu_1,\ldots,\nu_k\in\mathcal{P}_2(\R^d)$ and a vector of assignments $a\in[k]^N$. For measures that are not fully observed, i.e.\ for $i\in I_m$ (see \eqref{def_of_index_sets_missing_full} for the definition), we will use the fully observed measures in the same cluster, i.e.\ $\mu_\ell$ with $\ell\in J_i$, to impute $\push{P_i}{\mu_i}$ in a randomized fashion. Therefore, we have to interpret \eqref{eq:fill_up_point} in the current setting. Specifically, for $\mu_\ell$ with $\ell\in J_i$ we need to find
\begin{equation}\label{eq:imputed_measure}
    \eta_\ell^i\in\argmin\limits_{\substack{\eta \in \cP_2(\R^p)\\ \push{P_i}{\eta} = \push{P_i}{\mu_i}}}W_2^2(\eta, \mu_\ell).
\end{equation}
To obtain such $\eta_\ell^i$ we will apply the following lemma.
\begin{lemma}\label{lem:imputed_measure}
    Let $\mu_1\in\cP_2(\R^k),\nu\in\P_2(\R^d)$ with $d>k$. Denote by $\nu_1$ the projection on the first $k$ coordinates and by $\pi_1^\ast\in\Pi(\mu_1,\nu_1)$ the optimal transport coupling between $\mu_1$ and $\nu_1$. Then, the measure
    \begin{equation}\label{eq:lem_solution_to_projection_problem}
        \eta^\ast(\mathrm{d}x_1,\mathrm{d}x_2):=\int_{Y_1}\nu_{y_1}(\mathrm{d}x_2)\,\pi_1^\ast(\mathrm{d}x_1,\mathrm{d}y_1),
    \end{equation}
    solves
    \begin{equation}\label{eq:lemma_projection_problem}
        \min\limits_{\substack{\eta\in\cP_2(\R^d)\\ \eta_1=\mu_1}}W_2^2(\eta,\nu) = W_2^2(\mu_1,\nu_1).
    \end{equation}
\end{lemma}
\begin{proof}
    Let us first note that for any $\eta\in\cP_2(\R^d)$ with $\eta_1=\mu_1$ there is the trivial bound
    \begin{equation}
        W_2^2(\eta,\nu)\geq W_2^2(\eta_1,\nu_1) = W_2^2(\mu_1,\nu_1).
    \end{equation}
    We will show that $\eta^\ast$, as defined in \eqref{eq:lem_solution_to_projection_problem} achieves this lower bound, which will imply its optimality. In order to do this, we define a coupling $\pi\in\Pi(\eta^\ast,\nu)$ having those costs. In the following we split a point $z\in\R^d$ like $z=(z_1,z_2)\in\R^k\times\R^{d-k}$ and use the notations $X_1,X_2,Y_1,Y_2$ to emphasize over which space we are integrating, even though $X_1=Y_1=\R^k$ and $X_2=Y_2=\R^{d-k}$. Let us define $\pi$ by
    \begin{equation}
        \pi(\d x_1,\d y_1,\d x_2,\d y_2):=\pi_1^\ast(\d x_1,\d y_1)\nu_{y_1}(\d y_2)\delta_{y_2}(\d x_2).
    \end{equation}
    To see that this is actually a coupling of $(\eta^\ast,\nu)$, note that for the second marginal, we have
    \begin{align*}
        \int_{X_1\times X_2}\,\pi_1^\ast(\d x_1,\d y_1)\,\nu_{y_1}(\d y_2)\,\delta_{y_2}(\d x_2) &= \int_{X_1}\int_{X_2}\,\delta_{y_2}(\d x_2)\,\pi_1^\ast(\d x_1,\d y_1)\,\nu_{y_1}(\d y_2)\\
        & =\int_{X_1}\,\pi_1^\ast(\d x_1,\d y_1)\,\nu_{y_1}(\d y_2) \\
        &= \nu_1(\d y_1)\nu_{y_1}(\d y_2) = \nu(\d y_1,\d y_2).
    \end{align*}
    
    For the first marginal, we have
    \begin{align*}
        \int_{Y_1\times Y_2}\,\pi_1^\ast(\d x_1,\d y_1)\,\nu_{y_1}(\d y_2)\,\delta_{y_2}(\d x_2) &= \int_{Y_1}\int_{Y_2}\delta_{y_2}(\d x_2)\,\nu_{y_1}(\d y_2)\,\pi_1^\ast(\d x_1,\d y_1)\\
        &= \int_{Y_1}\nu_{y_1}(\d x_2)\,\pi_1^\ast(\d x_1,\d y_1) = \eta^\ast(\d x_1,\d x_2).
    \end{align*}

    Next, let us calculate the coupling's cost,
    \begin{align*}
        \int_{X\times Y}\lvert x-y\rvert^2\,\pi(\d x,\d y) &= \int_{X_1\times Y_1}\lvert x_1-y_1\rvert^2\,\pi_1^\ast(\d x_1,\d y_1) + \int_{X\times Y}\lvert x_2-y_2\rvert^2\, \pi(\d x, \d y) \\
        &=W_2^2(\mu_1,\nu_1) +\int_{X\times Y}\lvert x_2-y_2\rvert^2\,\pi(\d x,\d y),
    \end{align*}
    and note for the second term,
    \begin{align*}
        \int_{X\times Y}\lvert x_2-y_2\rvert^2\,\pi(\d x,\d y) &= 
        \int_{X_1\times Y_1}\int_{Y_2}\int_{X_2}\lvert x_2-y_2\rvert^2\,\delta_{y_2}(\d x_2)\,\nu_{y_1}(\d y_2)\,\pi_1^\ast(\d x_1,\d y_1)\\
        &= 0.
    \end{align*}
    Therefore, $W_2(\eta^\ast,\nu) = W_2(\mu_1,\nu_1)$, showing that $\eta^\ast$ solves \eqref{eq:lemma_projection_problem}.
\end{proof}
\begin{remark}
    If $\pi_1^\ast$ in \Cref{lem:imputed_measure} is of Monge-type, i.e.\ if there is a measurable map $T_1:\R^k\to\R^k$ such that $\pi_1^\ast = \push{(\mathrm{id}_{\R^k},T_1)}{\mu_1}$, then \eqref{eq:lem_solution_to_projection_problem} simplifies to
    \begin{equation}\label{eq:solution_to_projection_problem_Monge}
        \eta^\ast(\d x_1,\d x_2)=\mu_1(\d x_1)\nu_{T_1(x_1)}(\d x_2).
    \end{equation}
\end{remark}

Lemma \ref{lem:imputed_measure} describes the form of $\eta_\ell^i$ in \eqref{eq:imputed_measure}. To apply it, let $\pi_\ell^i\in\Pi(\push{P_i}{\mu_i},\push{P_i}{\mu_\ell})$ be the optimal coupling between $\push{P_i}{\mu_i}$ and $\push{P_i}{\mu_\ell}$. Additionally, let $\mathcal{I}_i\subset[d]$ denote the indices of the projection $P_i$, i.e.\ $P_i(x_1,\ldots,x_d) = (x_j)_{j\in\mathcal{I}_i}$, and for the remaining indices define $\mathcal{I}_i^C:=[d]\setminus \mathcal{I}_i$. A solution to \eqref{eq:imputed_measure} then is
\begin{align}
    \eta_\ell^i(\d x_1,\ldots,\d x_d) := \int_{\R^{\lvert\mathcal{I}_i^C\rvert}}{\mu_\ell}_{(y_j)_{j\in\mathcal{I}_i}}((\d x_j)_{j\in\mathcal{I}_i^C})\, \pi_\ell^i((\d x_j,\d y_j)_{j\in\mathcal{I}_i}).
\end{align}
Having this, we can translate \eqref{eq:imputed_as_measure_abstract_metric_space} to the current setting by using $\eta_\ell^i$ with $\ell\in J_i$, and define
\begin{align}\label{eq:imputed_random_measure} 
\mathbb{P}_i:=\sum_{\ell\in J_i}p_\ell^i\boldsymbol{\delta}_{\eta_\ell^i}
\end{align}
to obtain a measure $\mathbb{P}_i\in\mathcal{P}_2(\mathcal{P}_2(\R^d))$, i.e.\ a random measure. 
For fully observed measures, i.e.\ $i\in I_f$ we set $\mathbb{P}_i:=\boldsymbol{\delta}_{\mu_i}$. For $i\in I_m$ and $\lvert J_i\rvert > 1$ we suggest the weights
\begin{align*} 
p_\ell^i\propto\exp \big(-\frac{\lambda}{2\sigma_i^2}W_2^2(\push{P_i}{\mu_i},\push{P_i}{\mu_\ell})\big),~~~\ell\in J_i.
\end{align*}
This reflects the idea that measures which are closer to the to be imputed measure in the known coordinates should receive more weight. Here, $\lambda > 0$ is a tuning parameter that controls this weighting and the variance $\sigma_i^2$ is used to standardize the distances, i.e.\
\begin{align*} 
\sigma_i^2:=\frac{1}{\lvert J_i\rvert -1}\sum_{\ell\in J_i}W_2^2(\push{P_i}{\mu_i},\push{P_i}{\mu_\ell}).
\end{align*}
If $\lvert J_i\rvert=1$ it is clear what to do, as there is only one measure available for imputation and for $\lvert J_i\rvert = 0$ we use the barycenter $\nu_{a_i}$ to impute $\push{P_i}{\mu_i}$. 

After imputation we obtain random measures $\mathbb{P}_1,\ldots,\mathbb{P}_N\in\mathcal{P}_2(\mathcal{P}_2(\R^d))$. To calculate their pairwise distances, we use the metric $\rho$ defined in \eqref{prod_dist_general}. That is, for $\mathbb{P},\mathbb{Q}\in\mathcal{P}_2(\mathcal{P}_2(\R^d))$, $\rho(\mathbb{P},\mathbb{Q})=0$, if $\mathbb{P}=\mathbb{Q}$, and for $\mathbb{P}\neq\mathbb{Q}$
\begin{align}\label{distance_on_randomized_measures} 
\rho(\mathbb{P},\mathbb{Q})=\int_{\mathcal{P}_2(\R^d)}\int_{\mathcal{P}_2(\R^d)}W_2(\mu,\nu)\,\mathbb{P}(\d\mu)\,\mathbb{Q}(\d\nu).
\end{align}

\section{Evaluation Method}\label{sec:GW_evaluation}

{Our main motivation for NA Wasserstein $k$-means was to develop a method for clustering and reconstructing measures that may have missing coordinates. To the best of our knowledge, this is the first method to address this problem, so we cannot compare it to existing methods. In this short section we propose an abstract methodology for evaluating the quality of a reconstruction method in the setting of metric (measure) spaces, based on the concept of Gromov-Wasserstein distance. As just explained this evaluation methodology will not be applied directly to our main problem of interest, since there is no alternative method to compare to. However, we will apply it to a related problem in \Cref{sec:GMM_Rec}, where the data consists of Euclidean points rather than measures. For the Euclidean case we can compare NA $k$-means, a special case of NA Wasserstein $k$-means, to existing alternative methods. }

Suppose there are elements $x_1,\ldots,x_N$ in a metric space $(\X,d_{\X})$, which we collect in the set $X:=\{x_1,\ldots,x_N\}$. The data we observe are these points but only via the maps $\varphi_i:\X\to\X_i$ defined in Section \ref{sec:clustering_proj_metric}, i.e.\ we observe $\tilde{x}_i=\varphi_i(x_i)$. The observed points are collected in $(\tilde{x}_1,\ldots,\tilde{x}_N)\in\mathcal{X}_{NA}$, where we define $\X_{NA}: = \X_1\times\cdots\times\X_N.$ This can be summarized by saying that we only observe the values of $(x_1,\dots, x_N)$ under the map $h:\mathcal{X}^N\to\mathcal{X}_{NA}$, in particular $h_i(x_1,\ldots,x_N) =\varphi_i(x_i) = \tilde{x}_i$. 

The goal is to reconstruct the metric structure of $X=\{x_1,\ldots,x_N\}$, which is described by their respective pairwise distances, as accurately as possible. This means that we want to find a \emph{good} metric space $(\Y,d_{\Y})$ and a reconstruction map $R:\mathcal{X}_{NA}\to \mathcal{Y}^N$. The reconstructed point corresponding to $x_i$ then is $R_i(\tilde{x}_1,\ldots,\tilde{x}_N) =y_i$.

Before establishing a way of comparing the reconstructed points $R((\tilde{x}_1,\ldots,\tilde{x}_N))=(y_1,\ldots,y_N)$ to the original points $X$, we want to consider another piece of information, namely that some observations of $X$ might be more important for us than others. This is motivated by our main intended application of reconstructing a banking landscape, in which a bank's importance may be linked to attributes such as its size, measured, for example, by its total assets. This follows the rationale that for a regulator it may be more important to reconstruct the characteristics of a very large bank than those of a smaller bank. A simple way to account for this is to assign weights proportional to the (or a function of) size of the bank. Formally, we define a probability measure $\mu_X$ on $X$ by setting $\mu_X(\{x_i\}):=p_i\geq 0$, where we assume $\sum_{i=1}^Np_i = 1$. %Therefore we are given the metric measure space $(X,d_{\X},\mu_X)$. On the finite metric space $(R((\tilde{x}_1,\ldots,\tilde{x}_N)),d_{\Y})$ we consider the analogue 
Naturally we define $\mu_Y$, a probability measure on $\Y$, via $\mu_Y(\{y_i\}) = \mu_Y(R_i(\tilde{x}_1,\ldots,\tilde{x}_N)) := p_i$.

M{\'e}moli \cite{Me11}, and the related work of Sturm \cite{Stu20}, introduced the \emph{Gromov-Wasserstein distance} $GW_2$ in order to compare metric measure spaces. For two metric measure spaces $(\X,d_\X,\mu)$ and $(\Y,d_\Y,\nu)$, it is defined as 
\begin{align*} 
GW_2((\X,d_\X,\mu),(\Y,d_\Y,\nu))^2: =\\ \inf_{\pi\in\Pi(\mu,\nu)}\int\limits_{\X^2\times \Y^2}\lvert d_{\X}(x,x') - d_{\Y}(y,y')\rvert^2\,\pi(\d x,\d y)\,\pi(\d x',\d y').
\end{align*}

We will use the Gromov-Wasserstein distance to evaluate a reconstruction, i.e.\ we will use
\begin{align}\label{eq:GW_as_evaluation_criterion}
    GW_2\big((\X,d_\X,\mu_X),(\{R_1(\tilde{x}_1,\ldots,\tilde{x}_N),\ldots,R_N(\tilde{x}_1,\ldots,\tilde{x}_N)\},d_\Y,\nu_Y)\big)
\end{align}
as performance measure. Note, that in the case of reconstructing as described in \Cref{sec:imputing_general}, we have $\Y = \mathcal{P}(\X)$. For algorithms to compute the Gromov-Wasserstein distance we refer to \cite{Me11} and \cite{PeCuSo16}, which are implemented in Python optimal transport library \texttt{POT} (see \cite{FlCoGrAlBoChChCoFaFoGaGaJaRaReRoScSeSuTaToVa21}). We will use this implementation in \Cref{sec:GMM_Rec}. The curious reader is also referred to \cite{AlJa18,FiCoTaFl19,PeCuSo16} for some machine learning applications of Gromov-Wasserstein distances and related metrics.

\section{Experimental Results}\label{sec:exp_results}

\subsection{Reconstructing Points from a Gaussian Mixture Model}\label{sec:GMM_Rec}

We start by comparing our clustering / reconstruction method with existing imputation methods in the Euclidean case, i.e.\ $\X=\R^d$, and evaluate the results using the Gromov-Wasserstein distance, cf.\ \eqref{eq:GW_as_evaluation_criterion}, introduced in Section \ref{sec:GW_evaluation}.

We simulate data from a Gaussian Mixture Model 
\begin{equation}\label{eq:def_gamma_GMM}
    \gamma := \sum_{j=1}^k\alpha_j\mathcal{N}(\mu_j,\Sigma_j),
\end{equation}
with weights $\alpha_1,\ldots,\alpha_k\geq 0$ and $\sum_{j=1}^k \alpha_j=1$. Additionally, to simulate the importance of points, i.e.\ $\mu_X(\{x_i\})$ from Section \ref{sec:GW_evaluation}, we draw samples from a $\text{Lognormal}(\mu,\sigma^2)$ distribution and assign each observation a weight proportional to its sampled value. In our simulation study we fix the parameters $k=5$, $d=5$, and $N=500$, meaning that we always sample $500$ points. To simulate missing values we employ various missingness structures, corresponding to different choices of the map $h$ from Section \ref{sec:GW_evaluation}. Little and Rubin \cite{LiRu19} classify missing data mechanisms into the following three categories:
\begin{itemize}
    \item \emph{MCAR} (missing completely at random): The probability of a missing value does not depend on any observed or unobserved values.
    \item \emph{MAR} (missing at random): The probability of a missing value depends only on observed data, not on the missing data itself.
    \item \emph{MNAR} (missing not at random): The probability of a missing value depends on the missing values themselves, even when accounting for observed data.
\end{itemize}
In order to achieve robust results we incorporate different combinations of all these missingness mechanisms in our simulation study, i.e.\ in our definition of $h$ from \Cref{sec:GW_evaluation}.

Before describing the method used to create missing values, we first explain how the parameters for the data-generating process are chosen, i.e.\ $\alpha_1,\ldots,\alpha_k,\mu_1,\ldots,\mu_k,\Sigma_1,\ldots,\Sigma_k$ for the Gaussian Mixture Model, and $\mu,\sigma^2$ for the lognormal distribution governing the importance of sampled points. We set $\mu = 20,\sigma = 1.5$, as these parameters (empirically) model the total assets of a financial institution reasonably well. For the Gaussian Mixture Model, we aim to simulate the weights of the normal distributions, means and covariance matrices in a non-informative manner. Specifically, the weights $\alpha_j$ are sampled uniformly from the probability simplex, the means $\mu_j$'s are sampled uniformly from the cube $[-5,5]^{d}$ and the covariance matrices $\Sigma_j$'s are drawn from a Wishart distribution with parameters $(d,\mathrm{Id}_{d}/d)$. Note that with this choice of parameters for the Wishart distribution, we have $\E[\Sigma_j] = \mathrm{Id}_d$.

In order to create missing data, we aim to combine the types of missingness mechanisms described above. To achieve this, we choose weights $\beta^{MCAR},\beta^{MAR},\beta^{MNAR}\geq 0$ with $\beta^{MCAR} + \beta^{MAR} + \beta^{MNAR} = 1$. We then fix a proportion $p\in (0,1)$ which determines the total fraction of missing values. Given $N=500$ points sampled independently from $\gamma$, we compute for each dimension the $p\beta^{MNAR}$ quantile and proceed to create missing values as follows:
\begin{enumerate}[i)]
    \item \emph{MCAR}: set $100p\beta^{MCAR}\%$ values missing completely at random,
    \item \emph{MNAR}: for each dimension set all values to missing which are below the corresponding $p\beta^{MNAR}$ quantile,
    \item \emph{MAR}: for points where the first coordinate is observed and non-negative, the probability of missing values in the other coordinates is four times higher than for points where the first coordinate is observed and negative. In total, we create $100p\beta^{MAR}\%$ of missing values in the data following this rule.
\end{enumerate}
This procedure defines the map $h$ from \Cref{sec:GW_evaluation}. To generate the results in \Cref{table:results_GMM}, we set $p = 0.15$. For $\beta^{MCAR}, \beta^{MAR},$ and $\beta^{MNAR}$, we consider seven different combinations by applying each type of missingness individually, in pairs, and all three types together.

For each choice of $\beta^{MCAR}, \beta^{MAR},\beta^{MNAR}$, we sample the described parameters of the Gaussian Mixture Model $100$ times. We generate data, and compute pairwise distances between data points. Then, we create missing values according to $h$ and apply different imputation procedures presented below. If missing values are imputed by points, we compute pairwise distances between the imputed points. If the imputation is done via measures, as abstractly defined in \eqref{eq:imputed_as_measure_abstract_metric_space}, we use the distance $\rho$ from \Cref{sec:imputing_general} to compute pairwise distances. Finally, we calculate the corresponding Gromov-Wasserstein distance between original pairwise distances and the imputed pairwise distances, using the corresponding weights derived from the $\text{Lognormal}(\mu,\sigma^2)$ samples, i.e.\ we compute \eqref{eq:GW_as_evaluation_criterion}. In our experiment we look at the following imputation techniques:
\begin{itemize}
\item NA $k$-means: Our method in the case of Euclidean data, as outlined in Example \ref{ex:NA_KMeans}.
\item NA $k$-means-m: this corresponds to first applying NA $k$-means for clustering and imputation. After obtaining the measures $(\theta_i)_{i=1}^N$, we compute their expected values and use them as the imputed points.
\item Mean imputation: Each missing value is replaced by the mean of the corresponding attribute.
\item Median imputation: Each missing value is replaced by the median of the corresponding attribute.
\item Multiple imputation: A method \cite{Ru04,AzStFrLe11} that imputes missing values by generating multiple possible values and combines results.
\item KNN: K-nearest-neighbor imputation, where each missing value is replaced by a weighted average of points that are close in the coordinates which are not missing. We choose $K=4$.
\item LR: Missing values are imputed by regressing on observed values and using the predicted values as imputations.
\end{itemize}
For all methods not introduced in this article, i.e.\ all except NA $k$-means and NA $k$-means-m, we use the implementations from \texttt{scikit-learn}, \cite{scikit-learn}.

In Table \ref{table:results_GMM}, we report the estimated Gromov Wasserstein Distances along with their standard errors, based on sampling $100$ times for each combination of $\beta^{MCAR},\beta^{MAR}$ and $\beta^{MNAR}$.

% \begin{table*}[t]
% \caption{Gromov Wasserstein Distance for Euclidean Simulation With $p=0.15$.}
% \resizebox{.975\textwidth}{!}{%
% \begin{tabular}{lllccccccc}\toprule%
% \bfseries $\beta^{MCAR}$ & \bfseries $\beta^{MAR}$ & \bfseries $\beta^{MNAR}$ & \bfseries NA $k$-means. & \bfseries NA $k$-means-m & \bfseries mean imp. & \bfseries median imp. & \bfseries multiple imp. & \bfseries KNN & \bfseries LR\\
% \midrule
% \vspace*{-11pt}
% \csvreader[head to column names]{GWN500_0.15.csv}{}%
% {\\ \round{\betaMCAR} & \round{\betaMAR} & \round{\betaMNAR} & \NAKMeans $~\pm$ \NAKMeansSE & \NAKMeansMean $\pm$ \NAKMeansMeanSE & \Mean $~\pm$ \MeanSE & \Median $~\pm$ \MedianSE & \MultImp $~\pm$ \MultImpSE & \KNN $~\pm$ \KNNSE & \LR $~\pm$ \LRSE}%
% \\\bottomrule
% \end{tabular}%
% }\label{table:results_GMM}
% \end{table*}

\begin{table*}[t]
\caption{Gromov Wasserstein Distance for Euclidean Simulation with $p=0.15$.}
\resizebox{.975\textwidth}{!}{%
\begin{tabular}{lllccccccc}\toprule%
\bfseries $\beta^{MCAR}$ & \bfseries $\beta^{MAR}$ & \bfseries $\beta^{MNAR}$ & \bfseries NA $k$-means. & \bfseries NA $k$-means-m & \bfseries mean imp. & \bfseries median imp. & \bfseries KNN & \bfseries multiple imp. & \bfseries LR\\
\midrule
\vspace*{-11pt}
\\ 1 & 0 & 0 & 0.525~$\pm$~0.016 & 0.549~$\pm$~0.021 & 1.28~$\pm$~0.024 & 1.28~$\pm$~0.024 & 0.643~$\pm$~0.014 & 0.724~$\pm$~0.02 & 0.734~$\pm$~0.02
\\ 0 & 1 & 0 & 0.517~$\pm$~0.012 & 0.574~$\pm$~0.03 & 1.296~$\pm$~0.034 & 1.296~$\pm$~0.034 & 0.619~$\pm$~0.021 & 0.791~$\pm$~0.029 & 0.802~$\pm$~0.033
\\ 0 & 0 & 1 & 1.534~$\pm$~0.048 & 1.592~$\pm$~0.052 & 2.316~$\pm$~0.053 & 2.316~$\pm$~0.053 & 1.792~$\pm$~0.047 & 1.719~$\pm$~0.051 & 1.743~$\pm$~0.053
\\ 0.5 & 0.5 & 0 & 0.51~$\pm$~0.017 & 0.513~$\pm$~0.016 & 1.286~$\pm$~0.029 & 1.286~$\pm$~0.029 & 0.627~$\pm$~0.019 & 0.703~$\pm$~0.019 & 0.709~$\pm$~0.021
\\ 0 & 0.5 & 0.5 & 1.035~$\pm$~0.036 & 1.12~$\pm$~0.04 & 1.86~$\pm$~0.033 & 1.86~$\pm$~0.033 & 1.308~$\pm$~0.034 & 1.306~$\pm$~0.039 & 1.316~$\pm$~0.042
\\ 0.5 & 0 & 0.5 & 1.012~$\pm$~0.03 & 1.106~$\pm$~0.039 & 1.883~$\pm$~0.033 & 1.883~$\pm$~0.033 & 1.347~$\pm$~0.04 & 1.311~$\pm$~0.038 & 1.335~$\pm$~0.038
\\ 0.333 & 0.333 & 0.333 & 0.872~$\pm$~0.035 & 0.918~$\pm$~0.037 & 1.702~$\pm$~0.028 & 1.702~$\pm$~0.028 & 1.125~$\pm$~0.028 & 1.126~$\pm$~0.028 & 1.145~$\pm$~0.033

\\\bottomrule
\end{tabular}%
}\label{table:results_GMM}
\end{table*}

We observe that, in the case of Euclidean points with missing values, the proposed method consistently outperforms classical imputation techniques when considering the Gromov-Wasserstein distance as evaluation measure. Specifically, whether we apply \emph{NA $k$-means} directly or include an additional averaging step, i.e.\ \emph{NA $k$-means-m}, both algorithms outperform the remaining five methods across all considered missing data scenarios, i.e.\ combinations of MCAR, MAR and MNAR. In the Supplementary Material we use the (adjusted) Rand index \cite{Ra71,HuAr86} as evaluation method when treating NA $k$-means purely as a clustering algorithm, rather than as a method to reconstruct the metric structure of objects, i.e.\ as imputation method. There we also compare it to $k$-pod which is consistently outperformed by NA $k$-means (see Table \ref{table:RS_GMM}). Moreover, we provide results for additional values of $p$, i.e.\ the total share of missing values in the observed data, to accompany and robustify our results. 

\subsection{Reconstructing Financial Institutions}\label{sec:recon_financial}

As described in the introduction, the primary motivation for developing NA Wasserstein $k$-means comes from clustering financial institutions based on granular loan data they are obliged to report to the central bank. This problem is of particular interest to regulators, as it enables a data-driven assessment of similarities and differences between financial institutions. If prior beliefs about similarities exist, this method provides a way to confirm or challenge them. In particular, if financial institutions have already been grouped based on prior knowledge, our approach can serve as a validation tool, highlighting institutions whose cluster assignment deviates from the expected grouping and may therefore warrant further investigation.

Before presenting an experiment of applying NA Wasserstein $k$-means, note that there is another experiment justifying the usage of NA Wasserstein $k$-means in \Cref{sec:Justifying_NA_W_KMeans}.

To demonstrate NA Wasserstein $k$-means in a practical example, we use real data from Oesterreichische Nationalbank, the central bank of Austria, and analyze loan data reported by $321$ financial institutions in Austria. In total, our dataset consists of $129230$ loans. Each loan is described by up to four attributes: interest rate, interest rate margin, probability of default and par value. In our analysis we apply a logarithmic transformation to par value and then standardize all four attributes across all loans. Additionally, we account for loan size by weighting each loan within a financial institution proportionally to the logarithm of its size. Among the $321$ institutions, $265$ reported all four attributes, $45$ institutions reported three and $11$ institutions reported only two. As a result, our observed data consists of $265$ probability measures in $\mathcal{P}_2(\R^4)$, $45$ probability measures in $\mathcal{P}_2(\R^3)$ and $11$ probability measures in $\mathcal{P}_2(\R^2)$. Ideally, the complete dataset would consist of $321$ probability measures in $\mathcal{P}_2(\R^4)$ but due to differences in reporting, missing values must be accounted for in practice. While meaningful imputation is already challenging for Euclidean data, it is even more difficult when dealing with measure-valued data.

Since prior imputation is not necessary for NA Wasserstein $k$-means, we can apply it directly as described in \Cref{sec:NA_WKMeans} and cluster the financial institutions into seven groups. After clustering, we use the imputation method described in the same section to obtain random measures $\mathbb{P}_1,\ldots,\mathbb{P}_{321}$ on $\R^4$, cf.\ \Cref{eq:imputed_measure}. Using the generalized distance $\rho$ on $\mathcal{P}_2(\mathcal{P}_2(\R^4))$, defined in \eqref{distance_on_randomized_measures}, we compute pairwise distances between the (imputed) financial institutions. Once the pairwise distances are computed, it is possible to visualize the reconstructed financial institutions and clustering results. To create such a banking landscape, dimensionality reduction techniques may be applied to approximate pairwise distances in a lower-dimensional space. Here, we choose \emph{Isomap} (cf.\ \cite{TeSiLa00}) over other methods, such as multidimensional scaling, as it is better suited to recover nonlinear manifolds. Applying Isomap to the pairwise distance matrix $(\rho(\mathbb{P}_i,\mathbb{P}_j))_{i,j=1}^{321}$ maps the financial institutions to $\R^3$ while preserving the overall distance structure as closely as possible.

%%%% NOTE: figure moved to the introduction and only stays there
% \begin{figure}[ht]
% \centering
% \begin{minipage}{0.5\linewidth}
%     \centering
%     \includegraphics[height=5.75cm]{landscape.PNG}
%     %\setlength{\foxsep}{0pt}\fbox{\includegraphics[height=8cm]{landscape.png}}
% \end{minipage}%
% \begin{minipage}{0.5\linewidth}
%     \centering
%     \includegraphics[height = 5.75cm]{landscape_view2.PNG}
% \end{minipage}
% \caption{3-d visualization of the Austrian Banking Landscape from two different perspectives.\\ Interactive plot: \url{https://lorenzriess.github.io/TGOFI_landscape.html}}
% \label{fig:landscape_labels}
% \end{figure}

Figure \ref{fig:landscape_labels} presents the result of this three-dimensional representation, i.e.\ a banking landscape with the corresponding clusters. Each point in the figure represents one of the $321$ financial institutions, and each color describes a cluster. A regulatory analyst can now use this three-dimensional representation to reassess prior beliefs about the banking landscape or existing groupings of financial institutions and to identify institutions of particular interest. Additionally to prior beliefs, the visualization may help analysts to further their understanding which financial institutions are similar to each other and also identify financial institutions which behave significantly different to others with respect to their loan structure. Detecting potential outliers is of particular interest.

We also note that, having imputed the measures, i.e.\ ``reconstructed'' financial institutions, in particular having computed their pairwise distances $(\rho(\mathbb{P}_i,\mathbb{P}_j))_{i,j=1}^{321}$, any distance-based learning algorithm can be applied. Specifically, distance-based outlier algorithms can be used to identify institutions that deviate significantly from the rest. 

To conclude this experiment, we emphasize that clustering probability measures directly using NA Wasserstein $k$-means can yield significantly different results compared to clustering their aggregated data, such as their expectations, with methods like NA $k$-means. To illustrate this, we computed the expected values of the $321$ probability measures considered in this section and applied NA $k$-means to the resulting points (with possibly missing values) in $\mathbb{R}^4$. Using the obtained cluster labels, we then computed generalized Wasserstein barycenters for these clusters. Notably, when these assignments and barycenters were used in the loss function \Cref{eq:Loss_NA_Wasserstein_kmeans}, the resulting value was $22.26\%$ higher than when clustering the probability measures directly using \Cref{alg:NA_WKMeans}. Furthermore, comparing the labels of the two clusterings using the adjusted Rand index\footnotemark[1] yielded a value of $0.2058$, indicating very poor agreement, as a value of $0$ would correspond to the expected agreement of a random assignment. This highlights the advantage of clustering in the space of probability measures rather than relying on simple data aggregation.

\subsection{Justifying NA Wasserstein \texorpdfstring{$k$}{k}-means}\label{sec:Justifying_NA_W_KMeans}

To justify the use of NA Wasserstein $k$-means, we carry out another experiment. It is similar to the one in Section \ref{sec:GMM_Rec}, where we artificially create missing values, but this time using \emph{distributional data}. To the best of our knowledge, no existing method can cluster probability measures that are only observed as push-forwards under projections. Therefore, we perform an experiment on the real loan data from the previous section, simulating missing values completely at random. 

Specifically, we consider $100$ financial institutions from the previous section that reported all four attributes of their credits, meaning there were no missing values. For each of these institutions we sample $100$ credits to reduce the support size and, consequently, the computational cost, as NA Wasserstein $k$-means will be applied multiple times to obtain reliable standard errors. The resulting $100$ probability measures on $\R^4$ contain no missing values, allowing us to cluster them into five groups using Wasserstein $k$-means. The assignments from this clustering serve as the ``ground truth'' for this experiment. 

Next, we create missing values completely at random: we choose a percentage $p\in(0,1)$ and a number of affected dimensions $d_m\in\{1,2\}$. Then, we randomly select $100p$ institutions and for each selected institution, we randomly choose $d_m$ of the $4$ dimensions and set them to NA. After simulating missing values in this way, we apply NA Wasserstein $k$-means and compare the resulting cluster assignments to the ground truth using the adjusted Rand index.\footnotemark[1]\footnotetext[1]{The Rand index \cite{Ra71} is a similarity measure that evaluates the agreement between two clusterings by considering agreement or non-agreement of all pairwise assignments. It takes the value $1$ for perfect agreement and $0$ for complete disagreement. The adjusted Rand index \cite{HuAr86} corrects the Rand index for chance by adjusting for the expected similarity under a random model. It ranges from $-0.5$ to $1$, where $1$ indicates perfect agreement, $0$ corresponds to random labeling, and negative values suggest less agreement than expected by chance.}

For each combination of $p$ and $d_m$, the simulation of missing values with subsequent clustering is repeated $100$ times. The means and standard errors of the resulting adjusted Rand indices are presented in \Cref{table:results_meas_data_adj}. They indicate that in most settings, the ``ground truth'' labels are well recovered. As expected, the adjusted Rand index decreases as $p$ and $d_m$ increase, i.e.\ as the proportion of missing data grows. Notably, when only one coordinate is missing for up to a quarter of the data, the adjusted Rand index remains above $0.87$. Even when two coordinates are missing, it only drops slightly below $0.8$ if more than $20\%$ of the measures are affected. We conclude that this experiment demonstrates the effectiveness of NA Wasserstein $k$-means as a clustering algorithm for probability measures that are only partially observed as push-forwards under projections.

% \begin{table*}[ht]
% \centering
% \caption{Adjusted Rand Index for Granular Data}
% \resizebox{.75\textwidth}{!}{%
% \begin{tabular}{lccccc}\toprule%
%  & \bfseries $p=0.05$ & \bfseries $p=0.1$ & \bfseries $p=0.15$ & \bfseries $p=0.2$ & \bfseries $p=0.25$ \\
% \midrule
% \vspace*{-11pt}
% \csvreader[head to column names]{ars100seedsJustify.csv}{}%
% {\\\dim & \pFive $~\pm$ \pFiveSe & \pTen $~\pm$ \pTenSe & \pFifteen $~\pm$ \pFifteenSe & \pTwenty $~\pm$ \pTwentySe & \pTwentyfive $~\pm$ \pTwentyfiveSe}%
% \\\bottomrule
% \end{tabular}%
% }\label{table:results_meas_data_adj}
% \end{table*}

\begin{table*}[ht]
\centering
\caption{Adjusted Rand Index for Granular Data}
\resizebox{.75\textwidth}{!}{%
\begin{tabular}{lccccc}\toprule%
 & \bfseries $p=0.05$ & \bfseries $p=0.1$ & \bfseries $p=0.15$ & \bfseries $p=0.2$ & \bfseries $p=0.25$ \\
\midrule
\vspace*{-11pt}
\\ $d_m = 1$ & 0.942~$\pm$~0.006 & 0.919~$\pm$~0.006 & 0.891~$\pm$~0.006 & 0.87~$\pm$~0.005 & 0.871~$\pm$~0.006
\\ $d_m = 2$ & 0.911~$\pm$~0.006 & 0.881~$\pm$~0.006 & 0.838~$\pm$~0.007 & 0.821~$\pm$~0.007 & 0.786~$\pm$~0.007

\\\bottomrule
\end{tabular}%
}\label{table:results_meas_data_adj}
\end{table*}

\bibliographystyle{abbrv}
\bibliography{literature}

\appendix

\section{Existence of ``barycenters''}\label{sec:ex_barys}

In this section we discuss the existence of minimizers in the general barycenter updating step \eqref{eq:general_bary_update} of Section \ref{sec:problem_algo_general_metric}. Thus, let us consider a metric space $(\X,d)$ and finitely many continuous maps into other metric spaces $\varphi_i:\X\to\X_i$. We observe the image points $\tilde{x}_i=\varphi_i(x_i)$ for $i\in[n]$. The question is: under which conditions is
\begin{align} \label{eq:general_frechet_mean_problem}
\inf\limits_{y\in\X}\sum_{i=1}^n \lambda_id(\varphi_i^{-1}(\tilde{x}_i),y)^2
\end{align}
attained? Here, $\lambda_i>0,\sum_{i=1}^n\lambda_i=1$ denote some convex weights. Since in this general formulation the maps $\varphi_i$ can be arbitrary, minimizers do not have to exist. However, if one of the maps $\varphi_i$ is the identity on $\X$, we can guarantee existence under mild assumptions.

\begin{lemma}\label{lemma:existence_barys}
Assume there exists a metrizable topology $\tau$ on $\X$ which is weaker than the topology induced by the metric $d$, and such that 
\begin{enumerate}[i)]
\item $\varphi_i^{-1}(\tilde{x}_i)$ is closed in $\tau$ for all $i\in[n]$,\label{asmp_one_lemma}
\item balls, i.e.\ $B_d(y,c):=\{x\in\X:d(x,y)\leq c\}$, are $\tau$ compact for all $y\in\X,c>0$,\label{asmp_two_leamma}
\item $d(\cdot,\cdot)$ is lower semi continuous w.r.t.\ the product topology $\tau\times\tau$,\label{asmp_three_lemma_lsc_of_metric}
\item there exists $j\in[n]$ such that $\varphi_j=Id_\X$.\label{asmp_four_lemma}
\end{enumerate}
Then \eqref{eq:general_frechet_mean_problem} admits a minimizer.
\end{lemma}
\begin{proof}
Let us first prove that for $d(\varphi_i^{-1}(\tilde{x}_i),y)$ there exists $z\in\varphi_i^{-1}(\tilde{x}_i)$ such that
\begin{align} \label{eq:last_point}
d(\varphi_i^{-1}(\tilde{x}_i),y) = \inf\limits_{x\in\varphi_i^{-1}(\tilde{x}_i)}d(x,y) = d(z,y).
\end{align}
To this end, take a sequence $\{x_n\}_{n\geq 1}\subset \varphi_i^{-1}(\tilde{x}_i)$ such that
\begin{align*} 
d(x_n,y)\searrow \inf\limits_{x\in\varphi_i^{-1}(\tilde{x}_i)}d(x,y).
\end{align*}
We then set $c:=d(x_1,y)$, so $\{x_n\}_{n\geq 1}\subset B_d(y,c)$. Due to $\tau$-compactness of $B_d(y,c)$ we obtain a subsequence $(x_{n_k})_{k\geq 1}$ such that $x_{n_k}\overset{\tau}{\to}z\in B_d(y,c)$. Since $\{x_{n_k}\}_{k\geq 1}\subset \varphi_i^{-1}(\tilde{x}_i)$ and $\varphi_i^{-1}(\tilde{x}_i)$ is closed in $\tau$, we also have $z\in\varphi_i^{-1}(\tilde{x}_i)$. By lower semicontinuity of $d$ w.r.t.\ $\tau\times\tau$ we obtain
\begin{align*} 
d(z,y)\leq\liminf_{k\to\infty}d(x_{n_k},y) = \inf\limits_{x\in\varphi_i^{-1}(\tilde{x}_i)}d(x,y),
\end{align*}
which proves the existence of $z\in\varphi_i^{-1}(\tilde{x}_i)$ such that $d(z,y)=d(\varphi_i^{-1}(\tilde{x}_i),y)$.

Let us next prove that for a set $A=\varphi_i^{-1}(\tilde{x}_i)$ the map $x\mapsto d(A,x)$ is almost lower semi continuous w.r.t.\ $\tau$. Precisely, let us show that for a sequence $\{x_n\}_{\geq 1}\subset\X$ such that $x_n\overset{\tau}{\to} x\in\X$ and $d(x_n,x)\leq c$ we have
\begin{align}\label{eq:almost_lsc} 
\liminf_{n\to\infty}d(A,x_n)\geq d(A,x).
\end{align}
Without loss of generality, assume that the left-hand side of \eqref{eq:almost_lsc} is finite, as otherwise there is nothing to prove. Therefore, we may assume the existence of a constant $c'>0$ such that $d(A,x_n)\leq c'$.
By \eqref{eq:last_point}, for $n\geq 1$ there exists $z_n\in A$ such that $d(A,x_n) = d(z_n,x_n)$. We then have
\begin{align*} 
d(z_n,x)\leq d(z_n,x_n)+d(x_n,x) = d(A,x_n) + d(x_n,x)\leq c' + c.
\end{align*}
By $\tau$ compactness of $B_d(x,c+c')$ and since $A$ is closed in $\tau$, there exists $z\in A\cap B_d(x,c+c')$ such that $z_n\overset{\tau}{\to} z$ up to switching to a subsequence. By lower semi continuity of $d$ w.r.t. $\tau\times\tau$, i.e.\ assumption \ref{asmp_three_lemma_lsc_of_metric}, we have
\begin{align*} 
\liminf_{n\to\infty} d(A,x_n) = \liminf_{n\to\infty} d(z_n,x_n) \geq d(z,x) \geq d(A,x),
\end{align*}
which proves \eqref{eq:almost_lsc}.

To prove the existence of minimizers for \eqref{eq:general_frechet_mean_problem} we set
\begin{align*} 
V:=\inf\limits_{y\in\X}\sum_{i=1}^n\lambda_id(\varphi_i^{-1}(\tilde{x}_i),y)^2
\end{align*}
and take a sequence $\{y_n\}_{n\geq 1}\subset\X$ such that $\sum_{i=1}^n \lambda_id(\varphi_i^{-1}(\tilde{x}_i),y_n)^2\searrow V$. We assume without loss of generality that $d(\tilde{x}_j,y_n)\leq 2V$, since $\varphi_j=Id_\X$, so we have $\{y_n\}_{n\geq 1}\subset B_d(\tilde{x}_j,2V)$. By $\tau$-compactness of $B_d(\tilde{x}_j,2V)$ we have the existence of a subsequence $\{y_{n_k}\}_{k\geq 1}\subset B_d(\tilde{x}_j,2V)$ such that $y_{n_k}\overset{\tau}{\to} y^\ast$ for some $y^\ast\in B_d(\tilde{x}_j,2V)$. %Since all the sets $\varphi_i^{-1}(\tilde{x}_i)$ are closed in $\tau$ (even w.r.t.\ $d$) as pre-images of closed sets (singletons) under continuous maps, 
We can now use the proved property \eqref{eq:almost_lsc} to obtain
\begin{align*} 
V &= \liminf_{k\to\infty} \sum_{i=1}^n \lambda_id(\varphi_i^{-1}(\tilde{x}_i),y_{n_k})^2\\
& \geq \sum_{i=1}^n \lambda_id(\varphi_i^{-1}(\tilde{x}_i),y^\ast)^2.
\end{align*}
Therefore, we have found $y^\ast\in\X$ such that
\begin{align*} 
\sum_{i=1}^n \lambda_id(\varphi_i^{-1}(\tilde{x}_i),y^\ast)^2 = \inf\limits_{y\in\X}\sum_{i=1}^n\lambda_id(\varphi_i^{-1}(\tilde{x}_i),y)^2,
\end{align*}
finishing the proof.
\end{proof}

In the proof we also showed the existence of minimizers in \eqref{eq:fill_up_point}, i.e.\ points used for ``filling up'', under the same assumptions. This is precisely \eqref{eq:last_point}.

Let us discuss the assumptions of Lemma \ref{lemma:existence_barys}. Note that assumptions \ref{asmp_one_lemma}, \ref{asmp_two_leamma} and \ref{asmp_three_lemma_lsc_of_metric} are fulfilled in our two important examples of $\X = \R^d$ and $\X = \mathcal{P}_2(\R^d)$. Indeed, for $\X=\R^d$ this is clear, whereas for $\X=\mathcal{P}_2(\R^d)$ the role of the metric $\tau$ is played by the metric induced by weak convergence of probability measures. For details we refer to \cite[Chapter $7$]{Villani}. As for
Assumption \ref{asmp_four_lemma}, this basically (in practical terms) means that for each cluster we need at least one fully observed data point. In our two examples, NA $k$-means and NA Wasserstein $k$-means, this would not even be necessary as in \cite[Proposition 3.1]{JuGoSa21} it is shown that existence of the generalized Wasserstein barycenter is always guaranteed as long as the maps $\varphi_i$ are linear. Still, it might be numerically advantageous to have at least one fully observed point in each cluster or include the previous barycenter with a small weight, cf.\ \eqref{eq:bary_update_with_weight}.

\newpage

\section*{Supplementary Material: Further Simulation Experiments}\label{sec:supplementary}
Let us here discuss an extended version of our experiment in \Cref{sec:GMM_Rec} to evaluate our algorithm. In Section \ref{sec:GMM_Rec} we evaluated our method in the Euclidean setting, NA $k$-means, as an imputation method. However, without the imputation step it can also be viewed as a clustering algorithm. In the case of Euclidean data it can be viewed as a clustering algorithm for points with missing values. Thus, we can also compare it to classical $k$-means after imputing the missing data with more standard methods for imputing missing values as used in \ref{sec:GMM_Rec}. To recall, the standard methods for imputing missing values we consider are mean imputation, median imputation, multiple imputation, K-nearest neighbor imputation, and imputation through linear regression. However, in the case of clustering Euclidean points with missing data we also consider the $k$-pod method which was introduced in \cite{ChChBa16}. These authors consider the same loss function as we do in Example \ref{ex:NA_KMeans}, i.e.\ \eqref{eq:NA-kmeans}, however they propose a different algorithm. To perform the comparison, we use the same simulation from Section \ref{sec:GMM_Rec}, i.e.\ sampling from a Gaussian Mixture Model with $k=5$ clusters in $5$ dimensions and different settings of missing values. We use precisely the same simulated data, and for each of those we apply classical $k$-means to cluster the points after imputing them with a standard imputation method. For our proposed method, i.e.\ NA $k$-means, and for the $k$-pod method, no imputation step is needed as these directly cluster points with missing values. 
To evaluate the clustering results we use the so-called Rand index, also known as the Rand score, introduced in \cite{Ra71}, as well as the adjusted Rand index (c.f.\ adjusted Rand score), introduced in \cite{HuAr86}. They both compare a baseline clustering (in our case, the labels from the normal distribution of the Gaussian Mixture Model to which an observation belongs) to a clustering of the data using an imputation and/or clustering algorithm. The Rand score takes values in $[0,1]$, with $0$ meaning nothing is clustered the same, and $1$ meaning the clusterings coincide except for renaming of clusters. In particular, a higher Rand score corresponds to a better clustering method. The adjusted Rand score is similar to the Rand score but adjusted for chance. It takes value between $-0.5$ and $1$ with $0$ indicating a random clustering. 

In Table \ref{table:RS_GMM} we report the corresponding mean Rand scores $\pm$ standard errors for the simulations of \Cref{sec:GMM_Rec}. In Table \ref{table:adj_RS_GMM} there are the corresponding adjusted Rand scores.

% \begin{table}[ht]
% \caption{Rand Scores for Euclidean Simulation With $p=0.15$.}
% \resizebox{.975\textwidth}{!}{%
% \begin{tabular}{lllccccccc}\toprule%
% \bfseries $\beta^{MCAR}$ & \bfseries $\beta^{MAR}$ & \bfseries $\beta^{MNAR}$ & \bfseries NA $k$-means & \bfseries mean imp. & \bfseries median imp. & \bfseries multiple imp. & \bfseries KNN & \bfseries LR & \bfseries $k$-pod \\
% \midrule
% \vspace*{-11pt}
% \csvreader[head to column names]{RSN500_0.15.csv}{}%
% {\\ \round{\betaMCAR} & \round{\betaMAR} & \round{\betaMNAR} & \NAKMeans $~\pm$ \NAKMeansSE & \Mean $~\pm$ \MeanSE & \Median $~\pm$ \MedianSE & \MultImp $~\pm$ \MultImpSE & \KNN $~\pm$ \KNNSE & \LR $~\pm$ \LRSE & \kpod $\pm$ \kpodSE}%
% \\\bottomrule
% \end{tabular}%
% }\label{table:RS_GMM}
% \end{table}

\begin{table}[ht]
\caption{Rand Scores for Euclidean Simulation With $p=0.15$.}
\resizebox{.975\textwidth}{!}{%
\begin{tabular}{lllccccccc}\toprule%
\bfseries $\beta^{MCAR}$ & \bfseries $\beta^{MAR}$ & \bfseries $\beta^{MNAR}$ & \bfseries NA $k$-means & \bfseries mean imp. & \bfseries median imp. & \bfseries KNN & \bfseries multiple imp. & \bfseries LR & \bfseries $k$-pod \\
\midrule
\vspace*{-11pt}
\\ 1 & 0 & 0 & 0.901~$\pm$~0.008 & 0.876~$\pm$~0.008 & 0.864~$\pm$~0.008 & 0.896~$\pm$~0.008 & 0.897~$\pm$~0.008 & 0.892~$\pm$~0.008 & 0.855~$\pm$~0.008
\\ 0 & 1 & 0 & 0.907~$\pm$~0.008 & 0.874~$\pm$~0.007 & 0.879~$\pm$~0.007 & 0.905~$\pm$~0.008 & 0.894~$\pm$~0.008 & 0.899~$\pm$~0.008 & 0.865~$\pm$~0.007
\\ 0 & 0 & 1 & 0.861~$\pm$~0.009 & 0.831~$\pm$~0.007 & 0.836~$\pm$~0.008 & 0.857~$\pm$~0.008 & 0.859~$\pm$~0.008 & 0.859~$\pm$~0.008 & 0.838~$\pm$~0.008
\\ 0.5 & 0.5 & 0 & 0.897~$\pm$~0.008 & 0.879~$\pm$~0.007 & 0.863~$\pm$~0.007 & 0.902~$\pm$~0.008 & 0.907~$\pm$~0.008 & 0.893~$\pm$~0.008 & 0.866~$\pm$~0.007
\\ 0 & 0.5 & 0.5 & 0.884~$\pm$~0.009 & 0.857~$\pm$~0.007 & 0.855~$\pm$~0.007 & 0.879~$\pm$~0.008 & 0.883~$\pm$~0.008 & 0.881~$\pm$~0.008 & 0.856~$\pm$~0.008
\\ 0.5 & 0 & 0.5 & 0.882~$\pm$~0.008 & 0.857~$\pm$~0.007 & 0.853~$\pm$~0.007 & 0.877~$\pm$~0.008 & 0.882~$\pm$~0.008 & 0.888~$\pm$~0.008 & 0.85~$\pm$~0.007
\\ 0.333 & 0.333 & 0.333 & 0.884~$\pm$~0.009 & 0.865~$\pm$~0.007 & 0.864~$\pm$~0.007 & 0.884~$\pm$~0.008 & 0.885~$\pm$~0.008 & 0.879~$\pm$~0.009 & 0.851~$\pm$~0.007

\\\bottomrule
\end{tabular}%
}\label{table:RS_GMM}
\end{table}

% \begin{table}[ht]
% \caption{Adjusted Rand Scores for Euclidean Simulation With $p=0.15$.}
% \resizebox{.975\textwidth}{!}{%
% \begin{tabular}{lllccccccc}\toprule%
% \bfseries $\beta^{MCAR}$ & \bfseries $\beta^{MAR}$ & \bfseries $\beta^{MNAR}$ & \bfseries NA $k$-means & \bfseries mean imp. & \bfseries median imp. & \bfseries multiple imp. & \bfseries KNN & \bfseries LR & \bfseries $k$-pod \\
% \midrule
% \vspace*{-11pt}
% \csvreader[head to column names]{ARSN500_0.15.csv}{}%
% {\\ \round{\betaMCAR} & \round{\betaMAR} & \round{\betaMNAR} & \NAKMeans $~\pm$ \NAKMeansSE & \Mean $~\pm$ \MeanSE & \Median $~\pm$ \MedianSE & \MultImp $~\pm$ \MultImpSE & \KNN $~\pm$ \KNNSE & \LR $~\pm$ \LRSE & \kpod $\pm$ \kpodSE}%
% \\\bottomrule
% \end{tabular}%
% }\label{table:adj_RS_GMM}
% \end{table}

\begin{table}[ht]
\caption{Adjusted Rand Scores for Euclidean Simulation With $p=0.15$.}
\resizebox{.975\textwidth}{!}{%
\begin{tabular}{lllccccccc}\toprule%
\bfseries $\beta^{MCAR}$ & \bfseries $\beta^{MAR}$ & \bfseries $\beta^{MNAR}$ & \bfseries NA $k$-means & \bfseries mean imp. & \bfseries median imp. & \bfseries KNN & \bfseries multiple imp. & \bfseries LR & \bfseries $k$-pod \\
\midrule
\vspace*{-11pt}
\\ 1 & 0 & 0 & 0.768~$\pm$~0.018 & 0.706~$\pm$~0.016 & 0.68~$\pm$~0.016 & 0.756~$\pm$~0.016 & 0.756~$\pm$~0.017 & 0.746~$\pm$~0.017 & 0.661~$\pm$~0.016
\\ 0 & 1 & 0 & 0.778~$\pm$~0.017 & 0.699~$\pm$~0.015 & 0.709~$\pm$~0.015 & 0.778~$\pm$~0.017 & 0.752~$\pm$~0.017 & 0.763~$\pm$~0.017 & 0.68~$\pm$~0.016
\\ 0 & 0 & 1 & 0.669~$\pm$~0.018 & 0.593~$\pm$~0.015 & 0.608~$\pm$~0.017 & 0.661~$\pm$~0.018 & 0.665~$\pm$~0.017 & 0.664~$\pm$~0.017 & 0.617~$\pm$~0.018
\\ 0.5 & 0.5 & 0 & 0.758~$\pm$~0.018 & 0.71~$\pm$~0.015 & 0.676~$\pm$~0.015 & 0.768~$\pm$~0.017 & 0.78~$\pm$~0.016 & 0.75~$\pm$~0.017 & 0.68~$\pm$~0.016
\\ 0 & 0.5 & 0.5 & 0.727~$\pm$~0.019 & 0.656~$\pm$~0.016 & 0.654~$\pm$~0.016 & 0.714~$\pm$~0.016 & 0.724~$\pm$~0.017 & 0.717~$\pm$~0.017 & 0.658~$\pm$~0.017
\\ 0.5 & 0 & 0.5 & 0.719~$\pm$~0.018 & 0.656~$\pm$~0.015 & 0.648~$\pm$~0.015 & 0.704~$\pm$~0.017 & 0.718~$\pm$~0.016 & 0.735~$\pm$~0.017 & 0.644~$\pm$~0.015
\\ 0.333 & 0.333 & 0.333 & 0.727~$\pm$~0.019 & 0.676~$\pm$~0.015 & 0.676~$\pm$~0.015 & 0.726~$\pm$~0.017 & 0.728~$\pm$~0.017 & 0.716~$\pm$~0.018 & 0.644~$\pm$~0.016

\\\bottomrule
\end{tabular}%
}\label{table:adj_RS_GMM}
\end{table}

We can see that the introduced method NA $k$-means either performs best or lies within the standard error of the best performing method for each setting of missing values when considering it solely as clustering algorithm for points with missing data. The naive approaches of mean and median imputation are outperformed. Also, the $k$-pod method, which uses the same loss function \eqref{eq:NA-kmeans} but a different algorithm than proposed here, is outperformed in each setting.

\subsection*{Generalizing our Results}
In our simulation study, the amount of missing data is controlled by the parameter $p$, which represents the share of missing values. In Section \ref{sec:GMM_Rec} and the previous paragraph, this parameter was fixed at $p=0.15$. However, to assess the robustness of our method, we now vary $p$ across different values. In Table \ref{table:GW_Euclidean_varying_NAs}, Table \ref{table:RS_varying_NAs} and Table \ref{table:adj_RS_varying_NAs} we present the corresponding results analogous to Table \ref{table:results_GMM}, Table \ref{table:RS_GMM} and Table \ref{table:adj_RS_GMM}, respectively, when letting $p\in\{0.1,0.2,0.25,0.3\}$.

\begin{table}[ht]
\caption{Gromov Wasserstein Distance for Euclidean Simulation With Varying Share of Missing Values.}
\resizebox{.975\textwidth}{!}{%
\begin{tabular}{lllccccccc}\toprule%
\bfseries $\beta^{MCAR}$ & \bfseries $\beta^{MAR}$ & \bfseries $\beta^{MNAR}$ & \bfseries NA $k$-means. & \bfseries NA $k$-means-m & \bfseries mean imp. & \bfseries median imp. & \bfseries KNN & \bfseries multiple imp. & \bfseries LR\\
\midrule
\multicolumn{10}{c}{\bfseries $p=0.1$} \\ \cline{1-10}
\\ 1 & 0 & 0 & 0.447~$\pm$~0.021 & 0.452~$\pm$~0.02 & 1.021~$\pm$~0.026 & 1.021~$\pm$~0.026 & 0.479~$\pm$~0.019 & 0.625~$\pm$~0.028 & 0.618~$\pm$~0.024
\\ 0 & 1 & 0 & 0.426~$\pm$~0.018 & 0.427~$\pm$~0.015 & 1.012~$\pm$~0.028 & 1.012~$\pm$~0.028 & 0.445~$\pm$~0.019 & 0.617~$\pm$~0.025 & 0.609~$\pm$~0.023
\\ 0 & 0 & 1 & 1.171~$\pm$~0.042 & 1.189~$\pm$~0.041 & 1.833~$\pm$~0.038 & 1.833~$\pm$~0.038 & 1.335~$\pm$~0.038 & 1.418~$\pm$~0.042 & 1.391~$\pm$~0.041
\\ 0.5 & 0.5 & 0 & 0.406~$\pm$~0.01 & 0.453~$\pm$~0.033 & 1.003~$\pm$~0.024 & 1.003~$\pm$~0.024 & 0.446~$\pm$~0.016 & 0.573~$\pm$~0.017 & 0.584~$\pm$~0.02
\\ 0 & 0.5 & 0.5 & 0.874~$\pm$~0.042 & 0.923~$\pm$~0.039 & 1.462~$\pm$~0.028 & 1.462~$\pm$~0.028 & 1.013~$\pm$~0.033 & 1.057~$\pm$~0.029 & 1.045~$\pm$~0.032
\\ 0.5 & 0 & 0.5 & 0.856~$\pm$~0.036 & 0.876~$\pm$~0.036 & 1.462~$\pm$~0.026 & 1.462~$\pm$~0.026 & 0.989~$\pm$~0.031 & 1.055~$\pm$~0.031 & 1.061~$\pm$~0.033
\\ 0.333 & 0.333 & 0.333 & 0.764~$\pm$~0.036 & 0.764~$\pm$~0.03 & 1.33~$\pm$~0.021 & 1.33~$\pm$~0.021 & 0.877~$\pm$~0.032 & 0.908~$\pm$~0.026 & 0.914~$\pm$~0.028

\\\midrule
\multicolumn{10}{c}{\bfseries $p=0.2$} \\ \cline{1-10}
\\ 1 & 0 & 0 & 0.605~$\pm$~0.02 & 0.626~$\pm$~0.021 & 1.547~$\pm$~0.03 & 1.547~$\pm$~0.03 & 0.822~$\pm$~0.016 & 0.857~$\pm$~0.024 & 0.879~$\pm$~0.03
\\ 0 & 1 & 0 & 0.67~$\pm$~0.036 & 0.697~$\pm$~0.031 & 1.589~$\pm$~0.039 & 1.589~$\pm$~0.039 & 0.778~$\pm$~0.023 & 0.914~$\pm$~0.026 & 0.907~$\pm$~0.026
\\ 0 & 0 & 1 & 1.884~$\pm$~0.055 & 1.918~$\pm$~0.055 & 2.757~$\pm$~0.062 & 2.757~$\pm$~0.062 & 2.204~$\pm$~0.054 & 2.068~$\pm$~0.061 & 2.062~$\pm$~0.061
\\ 0.5 & 0.5 & 0 & 0.601~$\pm$~0.032 & 0.618~$\pm$~0.032 & 1.502~$\pm$~0.03 & 1.502~$\pm$~0.03 & 0.779~$\pm$~0.02 & 0.851~$\pm$~0.023 & 0.847~$\pm$~0.023
\\ 0 & 0.5 & 0.5 & 1.208~$\pm$~0.036 & 1.286~$\pm$~0.041 & 2.201~$\pm$~0.036 & 2.201~$\pm$~0.036 & 1.612~$\pm$~0.036 & 1.515~$\pm$~0.04 & 1.493~$\pm$~0.038
\\ 0.5 & 0 & 0.5 & 1.24~$\pm$~0.043 & 1.3~$\pm$~0.04 & 2.206~$\pm$~0.037 & 2.206~$\pm$~0.037 & 1.644~$\pm$~0.04 & 1.513~$\pm$~0.043 & 1.511~$\pm$~0.043
\\ 0.333 & 0.333 & 0.333 & 1.006~$\pm$~0.034 & 1.057~$\pm$~0.036 & 2.041~$\pm$~0.034 & 2.041~$\pm$~0.034 & 1.396~$\pm$~0.037 & 1.314~$\pm$~0.041 & 1.33~$\pm$~0.043

\\\midrule
\multicolumn{10}{c}{\bfseries $p = 0.25$} \\ \cline{1-10}
\\ 1 & 0 & 0 & 0.714~$\pm$~0.025 & 0.73~$\pm$~0.025 & 1.783~$\pm$~0.033 & 1.783~$\pm$~0.033 & 1.023~$\pm$~0.027 & 0.977~$\pm$~0.022 & 0.987~$\pm$~0.023
\\ 0 & 1 & 0 & 0.783~$\pm$~0.035 & 0.818~$\pm$~0.031 & 1.826~$\pm$~0.044 & 1.826~$\pm$~0.044 & 0.893~$\pm$~0.027 & 1.062~$\pm$~0.028 & 1.081~$\pm$~0.033
\\ 0 & 0 & 1 & 2.279~$\pm$~0.06 & 2.324~$\pm$~0.061 & 3.177~$\pm$~0.069 & 3.177~$\pm$~0.069 & 2.593~$\pm$~0.059 & 2.414~$\pm$~0.07 & 2.409~$\pm$~0.069
\\ 0.5 & 0.5 & 0 & 0.695~$\pm$~0.026 & 0.676~$\pm$~0.021 & 1.742~$\pm$~0.034 & 1.742~$\pm$~0.034 & 0.922~$\pm$~0.026 & 0.981~$\pm$~0.031 & 0.994~$\pm$~0.033
\\ 0 & 0.5 & 0.5 & 1.481~$\pm$~0.044 & 1.534~$\pm$~0.045 & 2.544~$\pm$~0.047 & 2.544~$\pm$~0.047 & 1.92~$\pm$~0.045 & 1.759~$\pm$~0.048 & 1.744~$\pm$~0.046
\\ 0.5 & 0 & 0.5 & 1.397~$\pm$~0.038 & 1.48~$\pm$~0.043 & 2.574~$\pm$~0.048 & 2.574~$\pm$~0.048 & 1.929~$\pm$~0.041 & 1.743~$\pm$~0.048 & 1.741~$\pm$~0.047
\\ 0.333 & 0.333 & 0.333 & 1.201~$\pm$~0.035 & 1.313~$\pm$~0.055 & 2.343~$\pm$~0.039 & 2.343~$\pm$~0.039 & 1.638~$\pm$~0.035 & 1.501~$\pm$~0.038 & 1.493~$\pm$~0.039

\\\midrule
\multicolumn{10}{c}{\bfseries $p = 0.3$} \\ \cline{1-10}
\\ 1 & 0 & 0 & 0.758~$\pm$~0.028 & 0.8~$\pm$~0.024 & 2.009~$\pm$~0.035 & 2.009~$\pm$~0.035 & 1.099~$\pm$~0.018 & 1.095~$\pm$~0.026 & 1.085~$\pm$~0.022
\\ 0 & 1 & 0 & 0.935~$\pm$~0.032 & 1.013~$\pm$~0.041 & 2.071~$\pm$~0.05 & 2.071~$\pm$~0.05 & 1.015~$\pm$~0.028 & 1.256~$\pm$~0.033 & 1.239~$\pm$~0.033
\\ 0 & 0 & 1 & 2.719~$\pm$~0.069 & 2.78~$\pm$~0.069 & 3.569~$\pm$~0.073 & 3.569~$\pm$~0.073 & 2.988~$\pm$~0.063 & 2.756~$\pm$~0.074 & 2.757~$\pm$~0.074
\\ 0.5 & 0.5 & 0 & 0.816~$\pm$~0.038 & 0.819~$\pm$~0.033 & 1.955~$\pm$~0.039 & 1.955~$\pm$~0.039 & 1.059~$\pm$~0.02 & 1.076~$\pm$~0.025 & 1.082~$\pm$~0.025
\\ 0 & 0.5 & 0.5 & 1.585~$\pm$~0.045 & 1.635~$\pm$~0.045 & 2.846~$\pm$~0.051 & 2.846~$\pm$~0.051 & 2.105~$\pm$~0.046 & 1.883~$\pm$~0.045 & 1.898~$\pm$~0.047
\\ 0.5 & 0 & 0.5 & 1.6~$\pm$~0.047 & 1.673~$\pm$~0.047 & 2.858~$\pm$~0.051 & 2.858~$\pm$~0.051 & 2.147~$\pm$~0.044 & 1.92~$\pm$~0.05 & 1.92~$\pm$~0.05
\\ 0.333 & 0.333 & 0.333 & 1.351~$\pm$~0.043 & 1.393~$\pm$~0.04 & 2.656~$\pm$~0.044 & 2.656~$\pm$~0.044 & 1.838~$\pm$~0.038 & 1.662~$\pm$~0.045 & 1.647~$\pm$~0.04

\\\bottomrule
\end{tabular}%
}\label{table:GW_Euclidean_varying_NAs}
\end{table}

From Table \ref{table:GW_Euclidean_varying_NAs} we conclude that the proposed method NA $k$-means outperforms all the other considered methods independent of the parameters needed for the missing value generation, when considering the Gromov Wasserstein distance as evaluation measure.

\begin{table}[ht]
\caption{Rand Scores for Euclidean Simulation With Varying Shares of Missing Values.}
\resizebox{.975\textwidth}{!}{%
\begin{tabular}{lllccccccc}\toprule%
\bfseries $\beta^{MCAR}$ & \bfseries $\beta^{MAR}$ & \bfseries $\beta^{MNAR}$ & \bfseries NA $k$-means & \bfseries mean imp. & \bfseries median imp. & \bfseries KNN & \bfseries multiple imp. & \bfseries LR & \bfseries $k$-pod \\
\midrule
\multicolumn{10}{c}{$p = 0.1$} \\ \cline{1-10}
\\ 1 & 0 & 0 & 0.906~$\pm$~0.008 & 0.887~$\pm$~0.008 & 0.886~$\pm$~0.008 & 0.911~$\pm$~0.008 & 0.897~$\pm$~0.008 & 0.904~$\pm$~0.008 & 0.875~$\pm$~0.007
\\ 0 & 1 & 0 & 0.902~$\pm$~0.008 & 0.893~$\pm$~0.007 & 0.888~$\pm$~0.007 & 0.903~$\pm$~0.008 & 0.906~$\pm$~0.007 & 0.906~$\pm$~0.008 & 0.881~$\pm$~0.007
\\ 0 & 0 & 1 & 0.881~$\pm$~0.009 & 0.864~$\pm$~0.008 & 0.858~$\pm$~0.008 & 0.876~$\pm$~0.008 & 0.88~$\pm$~0.008 & 0.881~$\pm$~0.008 & 0.873~$\pm$~0.007
\\ 0.5 & 0.5 & 0 & 0.905~$\pm$~0.008 & 0.888~$\pm$~0.007 & 0.883~$\pm$~0.007 & 0.912~$\pm$~0.008 & 0.906~$\pm$~0.008 & 0.901~$\pm$~0.008 & 0.877~$\pm$~0.007
\\ 0 & 0.5 & 0.5 & 0.896~$\pm$~0.009 & 0.874~$\pm$~0.007 & 0.875~$\pm$~0.007 & 0.897~$\pm$~0.008 & 0.885~$\pm$~0.008 & 0.895~$\pm$~0.008 & 0.875~$\pm$~0.008
\\ 0.5 & 0 & 0.5 & 0.892~$\pm$~0.009 & 0.878~$\pm$~0.008 & 0.869~$\pm$~0.007 & 0.894~$\pm$~0.008 & 0.891~$\pm$~0.008 & 0.898~$\pm$~0.008 & 0.875~$\pm$~0.008
\\ 0.333 & 0.333 & 0.333 & 0.896~$\pm$~0.009 & 0.876~$\pm$~0.008 & 0.886~$\pm$~0.007 & 0.905~$\pm$~0.008 & 0.897~$\pm$~0.009 & 0.905~$\pm$~0.008 & 0.883~$\pm$~0.007

\\\midrule
\multicolumn{10}{c}{$p = 0.2$} \\ \cline{1-10}
\\ 1 & 0 & 0 & 0.889~$\pm$~0.009 & 0.845~$\pm$~0.007 & 0.842~$\pm$~0.007 & 0.879~$\pm$~0.008 & 0.88~$\pm$~0.008 & 0.878~$\pm$~0.008 & 0.828~$\pm$~0.006
\\ 0 & 1 & 0 & 0.884~$\pm$~0.008 & 0.862~$\pm$~0.007 & 0.853~$\pm$~0.007 & 0.883~$\pm$~0.008 & 0.888~$\pm$~0.007 & 0.888~$\pm$~0.008 & 0.834~$\pm$~0.007
\\ 0 & 0 & 1 & 0.837~$\pm$~0.009 & 0.803~$\pm$~0.008 & 0.808~$\pm$~0.007 & 0.821~$\pm$~0.009 & 0.847~$\pm$~0.008 & 0.844~$\pm$~0.009 & 0.796~$\pm$~0.01
\\ 0.5 & 0.5 & 0 & 0.898~$\pm$~0.008 & 0.865~$\pm$~0.007 & 0.848~$\pm$~0.007 & 0.895~$\pm$~0.007 & 0.89~$\pm$~0.008 & 0.891~$\pm$~0.008 & 0.835~$\pm$~0.007
\\ 0 & 0.5 & 0.5 & 0.857~$\pm$~0.009 & 0.828~$\pm$~0.008 & 0.835~$\pm$~0.008 & 0.853~$\pm$~0.008 & 0.871~$\pm$~0.009 & 0.868~$\pm$~0.009 & 0.821~$\pm$~0.008
\\ 0.5 & 0 & 0.5 & 0.87~$\pm$~0.008 & 0.834~$\pm$~0.007 & 0.832~$\pm$~0.007 & 0.859~$\pm$~0.008 & 0.865~$\pm$~0.007 & 0.866~$\pm$~0.008 & 0.826~$\pm$~0.007
\\ 0.333 & 0.333 & 0.333 & 0.882~$\pm$~0.008 & 0.835~$\pm$~0.007 & 0.841~$\pm$~0.007 & 0.867~$\pm$~0.008 & 0.88~$\pm$~0.008 & 0.879~$\pm$~0.008 & 0.833~$\pm$~0.007

\\\midrule
\multicolumn{10}{c}{$p = 0.25$} \\ \cline{1-10}
\\ 1 & 0 & 0 & 0.875~$\pm$~0.009 & 0.824~$\pm$~0.007 & 0.82~$\pm$~0.006 & 0.85~$\pm$~0.008 & 0.867~$\pm$~0.008 & 0.867~$\pm$~0.008 & 0.802~$\pm$~0.007
\\ 0 & 1 & 0 & 0.873~$\pm$~0.008 & 0.839~$\pm$~0.007 & 0.836~$\pm$~0.007 & 0.873~$\pm$~0.008 & 0.873~$\pm$~0.007 & 0.874~$\pm$~0.008 & 0.817~$\pm$~0.008
\\ 0 & 0 & 1 & 0.815~$\pm$~0.009 & 0.793~$\pm$~0.008 & 0.793~$\pm$~0.009 & 0.785~$\pm$~0.008 & 0.821~$\pm$~0.009 & 0.826~$\pm$~0.009 & 0.77~$\pm$~0.01
\\ 0.5 & 0.5 & 0 & 0.891~$\pm$~0.008 & 0.845~$\pm$~0.007 & 0.841~$\pm$~0.007 & 0.87~$\pm$~0.008 & 0.88~$\pm$~0.007 & 0.881~$\pm$~0.008 & 0.817~$\pm$~0.007
\\ 0 & 0.5 & 0.5 & 0.851~$\pm$~0.009 & 0.805~$\pm$~0.007 & 0.805~$\pm$~0.007 & 0.829~$\pm$~0.008 & 0.854~$\pm$~0.008 & 0.856~$\pm$~0.008 & 0.794~$\pm$~0.008
\\ 0.5 & 0 & 0.5 & 0.855~$\pm$~0.008 & 0.81~$\pm$~0.007 & 0.812~$\pm$~0.007 & 0.827~$\pm$~0.008 & 0.855~$\pm$~0.009 & 0.856~$\pm$~0.008 & 0.796~$\pm$~0.008
\\ 0.333 & 0.333 & 0.333 & 0.862~$\pm$~0.009 & 0.822~$\pm$~0.007 & 0.825~$\pm$~0.007 & 0.843~$\pm$~0.008 & 0.864~$\pm$~0.008 & 0.864~$\pm$~0.009 & 0.806~$\pm$~0.008

\\\midrule
\multicolumn{10}{c}{$p = 0.3$} \\ \cline{1-10}
\\ 1 & 0 & 0 & 0.87~$\pm$~0.009 & 0.803~$\pm$~0.007 & 0.8~$\pm$~0.006 & 0.849~$\pm$~0.008 & 0.856~$\pm$~0.008 & 0.856~$\pm$~0.008 & 0.765~$\pm$~0.008
\\ 0 & 1 & 0 & 0.864~$\pm$~0.008 & 0.822~$\pm$~0.007 & 0.828~$\pm$~0.007 & 0.863~$\pm$~0.008 & 0.862~$\pm$~0.007 & 0.865~$\pm$~0.007 & 0.787~$\pm$~0.007
\\ 0 & 0 & 1 & 0.802~$\pm$~0.009 & 0.781~$\pm$~0.008 & 0.777~$\pm$~0.008 & 0.755~$\pm$~0.007 & 0.81~$\pm$~0.009 & 0.806~$\pm$~0.008 & 0.754~$\pm$~0.009
\\ 0.5 & 0.5 & 0 & 0.883~$\pm$~0.008 & 0.825~$\pm$~0.007 & 0.828~$\pm$~0.006 & 0.861~$\pm$~0.007 & 0.868~$\pm$~0.008 & 0.871~$\pm$~0.007 & 0.793~$\pm$~0.006
\\ 0 & 0.5 & 0.5 & 0.848~$\pm$~0.009 & 0.798~$\pm$~0.007 & 0.789~$\pm$~0.007 & 0.811~$\pm$~0.007 & 0.846~$\pm$~0.008 & 0.844~$\pm$~0.008 & 0.779~$\pm$~0.008
\\ 0.5 & 0 & 0.5 & 0.845~$\pm$~0.008 & 0.792~$\pm$~0.007 & 0.798~$\pm$~0.007 & 0.812~$\pm$~0.007 & 0.84~$\pm$~0.007 & 0.847~$\pm$~0.007 & 0.778~$\pm$~0.008
\\ 0.333 & 0.333 & 0.333 & 0.862~$\pm$~0.008 & 0.803~$\pm$~0.006 & 0.798~$\pm$~0.007 & 0.833~$\pm$~0.007 & 0.853~$\pm$~0.008 & 0.847~$\pm$~0.007 & 0.78~$\pm$~0.008

\\\bottomrule
\end{tabular}%
}\label{table:RS_varying_NAs}
\end{table}

\begin{table}[ht]
\caption{Adjusted Rand Scores for Euclidean Simulation With Varying Share of Missing Values.}
\resizebox{.975\textwidth}{!}{%
\begin{tabular}{lllccccccc}\toprule%
\bfseries $\beta^{MCAR}$ & \bfseries $\beta^{MAR}$ & \bfseries $\beta^{MNAR}$ & \bfseries NA $k$-means & \bfseries mean imp. & \bfseries median imp. & \bfseries KNN & \bfseries multiple imp. & \bfseries LR & \bfseries $k$-pod \\
\midrule
\multicolumn{10}{c}{$p = 0.1$} \\ \cline{1-10}
\\ 1 & 0 & 0 & 0.779~$\pm$~0.017 & 0.733~$\pm$~0.017 & 0.732~$\pm$~0.016 & 0.791~$\pm$~0.017 & 0.758~$\pm$~0.016 & 0.774~$\pm$~0.016 & 0.704~$\pm$~0.015
\\ 0 & 1 & 0 & 0.768~$\pm$~0.017 & 0.748~$\pm$~0.016 & 0.734~$\pm$~0.015 & 0.774~$\pm$~0.017 & 0.779~$\pm$~0.016 & 0.778~$\pm$~0.017 & 0.716~$\pm$~0.016
\\ 0 & 0 & 1 & 0.72~$\pm$~0.018 & 0.676~$\pm$~0.017 & 0.662~$\pm$~0.016 & 0.707~$\pm$~0.017 & 0.717~$\pm$~0.017 & 0.719~$\pm$~0.018 & 0.698~$\pm$~0.016
\\ 0.5 & 0.5 & 0 & 0.776~$\pm$~0.017 & 0.734~$\pm$~0.016 & 0.722~$\pm$~0.015 & 0.793~$\pm$~0.018 & 0.779~$\pm$~0.016 & 0.767~$\pm$~0.017 & 0.711~$\pm$~0.015
\\ 0 & 0.5 & 0.5 & 0.757~$\pm$~0.019 & 0.699~$\pm$~0.015 & 0.705~$\pm$~0.016 & 0.757~$\pm$~0.017 & 0.732~$\pm$~0.018 & 0.754~$\pm$~0.017 & 0.705~$\pm$~0.017
\\ 0.5 & 0 & 0.5 & 0.744~$\pm$~0.018 & 0.712~$\pm$~0.016 & 0.687~$\pm$~0.015 & 0.752~$\pm$~0.018 & 0.743~$\pm$~0.017 & 0.76~$\pm$~0.018 & 0.705~$\pm$~0.018
\\ 0.333 & 0.333 & 0.333 & 0.755~$\pm$~0.018 & 0.703~$\pm$~0.016 & 0.73~$\pm$~0.014 & 0.776~$\pm$~0.017 & 0.76~$\pm$~0.018 & 0.776~$\pm$~0.018 & 0.724~$\pm$~0.016

\\\midrule
\multicolumn{10}{c}{$p = 0.2$} \\ \cline{1-10}
\\ 1 & 0 & 0 & 0.741~$\pm$~0.018 & 0.629~$\pm$~0.014 & 0.622~$\pm$~0.014 & 0.718~$\pm$~0.016 & 0.72~$\pm$~0.016 & 0.714~$\pm$~0.017 & 0.593~$\pm$~0.015
\\ 0 & 1 & 0 & 0.728~$\pm$~0.017 & 0.665~$\pm$~0.015 & 0.649~$\pm$~0.014 & 0.723~$\pm$~0.016 & 0.736~$\pm$~0.015 & 0.733~$\pm$~0.016 & 0.608~$\pm$~0.015
\\ 0 & 0 & 1 & 0.614~$\pm$~0.019 & 0.528~$\pm$~0.016 & 0.537~$\pm$~0.016 & 0.572~$\pm$~0.018 & 0.634~$\pm$~0.017 & 0.629~$\pm$~0.018 & 0.521~$\pm$~0.02
\\ 0.5 & 0.5 & 0 & 0.759~$\pm$~0.017 & 0.674~$\pm$~0.015 & 0.634~$\pm$~0.014 & 0.75~$\pm$~0.015 & 0.739~$\pm$~0.017 & 0.742~$\pm$~0.017 & 0.608~$\pm$~0.016
\\ 0 & 0.5 & 0.5 & 0.664~$\pm$~0.018 & 0.592~$\pm$~0.016 & 0.605~$\pm$~0.017 & 0.652~$\pm$~0.016 & 0.697~$\pm$~0.018 & 0.691~$\pm$~0.017 & 0.576~$\pm$~0.018
\\ 0.5 & 0 & 0.5 & 0.691~$\pm$~0.017 & 0.598~$\pm$~0.015 & 0.596~$\pm$~0.014 & 0.663~$\pm$~0.016 & 0.677~$\pm$~0.016 & 0.682~$\pm$~0.017 & 0.588~$\pm$~0.016
\\ 0.333 & 0.333 & 0.333 & 0.719~$\pm$~0.017 & 0.6~$\pm$~0.014 & 0.616~$\pm$~0.015 & 0.683~$\pm$~0.017 & 0.714~$\pm$~0.016 & 0.71~$\pm$~0.016 & 0.598~$\pm$~0.015

\\\midrule
\multicolumn{10}{c}{$p = 0.25$} \\ \cline{1-10}
\\ 1 & 0 & 0 & 0.709~$\pm$~0.018 & 0.583~$\pm$~0.014 & 0.574~$\pm$~0.014 & 0.646~$\pm$~0.016 & 0.689~$\pm$~0.016 & 0.688~$\pm$~0.016 & 0.538~$\pm$~0.015
\\ 0 & 1 & 0 & 0.698~$\pm$~0.016 & 0.607~$\pm$~0.015 & 0.606~$\pm$~0.015 & 0.699~$\pm$~0.016 & 0.695~$\pm$~0.016 & 0.701~$\pm$~0.016 & 0.566~$\pm$~0.017
\\ 0 & 0 & 1 & 0.558~$\pm$~0.019 & 0.505~$\pm$~0.016 & 0.508~$\pm$~0.02 & 0.488~$\pm$~0.017 & 0.574~$\pm$~0.018 & 0.583~$\pm$~0.019 & 0.466~$\pm$~0.02
\\ 0.5 & 0.5 & 0 & 0.742~$\pm$~0.017 & 0.628~$\pm$~0.015 & 0.616~$\pm$~0.015 & 0.695~$\pm$~0.016 & 0.713~$\pm$~0.016 & 0.719~$\pm$~0.016 & 0.567~$\pm$~0.015
\\ 0 & 0.5 & 0.5 & 0.645~$\pm$~0.018 & 0.53~$\pm$~0.015 & 0.532~$\pm$~0.015 & 0.59~$\pm$~0.016 & 0.651~$\pm$~0.018 & 0.656~$\pm$~0.017 & 0.506~$\pm$~0.018
\\ 0.5 & 0 & 0.5 & 0.656~$\pm$~0.017 & 0.544~$\pm$~0.015 & 0.552~$\pm$~0.016 & 0.587~$\pm$~0.016 & 0.657~$\pm$~0.018 & 0.658~$\pm$~0.018 & 0.524~$\pm$~0.018
\\ 0.333 & 0.333 & 0.333 & 0.674~$\pm$~0.018 & 0.573~$\pm$~0.015 & 0.579~$\pm$~0.016 & 0.627~$\pm$~0.016 & 0.677~$\pm$~0.017 & 0.678~$\pm$~0.018 & 0.54~$\pm$~0.017

\\\midrule
\multicolumn{10}{c}{$p = 0.3$} \\ \cline{1-10}
\\ 1 & 0 & 0 & 0.695~$\pm$~0.017 & 0.528~$\pm$~0.013 & 0.525~$\pm$~0.014 & 0.641~$\pm$~0.015 & 0.661~$\pm$~0.014 & 0.662~$\pm$~0.015 & 0.462~$\pm$~0.014
\\ 0 & 1 & 0 & 0.674~$\pm$~0.016 & 0.571~$\pm$~0.014 & 0.588~$\pm$~0.015 & 0.672~$\pm$~0.016 & 0.667~$\pm$~0.014 & 0.674~$\pm$~0.015 & 0.503~$\pm$~0.016
\\ 0 & 0 & 1 & 0.522~$\pm$~0.019 & 0.474~$\pm$~0.015 & 0.467~$\pm$~0.018 & 0.407~$\pm$~0.015 & 0.541~$\pm$~0.019 & 0.532~$\pm$~0.018 & 0.425~$\pm$~0.018
\\ 0.5 & 0.5 & 0 & 0.722~$\pm$~0.016 & 0.577~$\pm$~0.014 & 0.588~$\pm$~0.015 & 0.669~$\pm$~0.015 & 0.687~$\pm$~0.016 & 0.691~$\pm$~0.015 & 0.51~$\pm$~0.014
\\ 0 & 0.5 & 0.5 & 0.634~$\pm$~0.018 & 0.507~$\pm$~0.014 & 0.491~$\pm$~0.015 & 0.543~$\pm$~0.014 & 0.629~$\pm$~0.016 & 0.626~$\pm$~0.017 & 0.478~$\pm$~0.017
\\ 0.5 & 0 & 0.5 & 0.624~$\pm$~0.017 & 0.492~$\pm$~0.014 & 0.515~$\pm$~0.017 & 0.542~$\pm$~0.015 & 0.609~$\pm$~0.016 & 0.628~$\pm$~0.017 & 0.469~$\pm$~0.017
\\ 0.333 & 0.333 & 0.333 & 0.667~$\pm$~0.017 & 0.519~$\pm$~0.013 & 0.511~$\pm$~0.015 & 0.594~$\pm$~0.015 & 0.646~$\pm$~0.016 & 0.629~$\pm$~0.015 & 0.474~$\pm$~0.016

\\\bottomrule
\end{tabular}%
}\label{table:adj_RS_varying_NAs}
\end{table}

From Tables \ref{table:RS_varying_NAs} and \ref{table:adj_RS_varying_NAs} we can see that also w.r.t.\ the (adjusted) Rand score, when NA $k$-means is solely viewed as clustering procedure for Euclidean points with missing values, it competes with the best performing imputation + clustering techniques. Notably, $k$-pod is outperformed consistently.

\end{document}